\documentclass[twoside]{article}

\usepackage[preprint]{aistats2026}

\usepackage[font=small,labelfont=bf]{caption}

\usepackage{amsmath}
\DeclareMathOperator{\Var}{Var}

\usepackage{booktabs}
\usepackage{amsfonts}   %
\usepackage{amssymb}    %
\usepackage{float}
\usepackage{amsmath,amssymb,amsfonts}

\usepackage{graphicx}
\usepackage{subcaption} %

\usepackage[table]{xcolor} %

\usepackage{graphicx}

\usepackage[round]{natbib}
\usepackage{graphicx}   %
\usepackage{subcaption} %
\graphicspath{{img/}}

\usepackage{booktabs}
\usepackage[table]{xcolor}
\usepackage{makecell}     %
\usepackage{array}
\usepackage{makecell}

\providecommand{\cref}[1]{#1}
\providecommand{\multirow}[3]{#3}
\providecommand{\url}[1]{#1}
\providecommand{\argmax}{}

\newtheorem{theorem}{Theorem}

\begin{document}

\twocolumn[

\aistatstitle{GeoTTER: Leveraging Local Geometry of Optimal Transport for Zero-Shot Classification}

\runningtitle{GeoTTER: Leveraging Local Geometry of Optimal Transport for Zero-Shot Classification (AISTATS'26)}

\aistatsauthor{Wei-Yang Alex Lee \And Rudrasis Chakraborty \And  Vishnu Lokhande }

\aistatsaddress{SUNY Buffalo \And  Lawrence Livermore National Lab \And SUNY Buffalo} ]

\begin{abstract}
We present GeoTTER, a novel framework that redefines optimal transport in the realm of zero-shot classification. Conventional methods often suffer from miscalibration and a lack of adaptability, as they rely on fixed cost matrices derived solely from pre-trained model embeddings. In contrast, GeoTTER addresses these limitations by incorporating two key techniques. First, to alleviate high-frequency label jaggedness (sample-level manifold jitter that assigns neighboring embeddings to different classes), GeoTTER integrates local geometric structure into the optimal transport formulation via graph-Laplacian smoothing, a technique grounded in spectral graph theory that enforces neighborhood consistency. Second, to correct coherent angular drift (a low-frequency orientation bias in which large groups of samples share the same angular offset from their true label prototypes), we fuse clustering-guided cost components with a globally adjusted transport cost, achieving a multi-objective optimization that respects both global distribution constraints and latent data structure. With a median improvement of +6.82\% compared to zero-shot and +2.13\% compared to OTTER, GeoTTER shows robust improvements across a diverse set of benchmarks.

\footnote{Paper accepted at Proceedings of the 29th International Conference on Artificial Intelligence and Statistics (AISTATS) 2026, Tangier, Morocco.}
\end{abstract}

\section{INTRODUCTION}

Optimal transport (OT) has emerged as a powerful and flexible tool for aligning distributions across modalities, showing promise in a wide range of tasks from domain adaptation to zero-shot learning ~\cite{cuturi2013sinkhorndistanceslightspeedcomputation}. However, despite its theoretical elegance, standard OT formulations often suffer from poor calibration, numerical instability, and limited capacity to model local structural nuances within the embedding space. These challenges become particularly pronounced in modern zero-shot and prompt-based classification settings, where the misalignment between label priors and prediction logits can significantly degrade performance.~\cite{liusie2023mitigatingwordbiaszeroshot}

Recent works have attempted to address some of these limitations. For instance, OTTER  ~\cite{shin2024ottereffortlesslabeldistribution} proposes a framework to adapt label distributions in zero-shot models without requiring retraining, highlighting the importance of flexible cost matrix designs. Liusie et al. ~\cite{liusie2023mitigatingwordbiaszeroshot} explore the problem of word bias in prompt-based zero-shot classifiers, advocating for more robust methods that incorporate semantic priors. Bai et al. ~\cite{bai2023adaptingonlinelabelshift} tackle the label shift problem in online settings, providing provable adaptation mechanisms. These lines of work all point to a common need: OT methods must become more structure-sensitive to generalize effectively especially when label distributions are different.

In this paper, we propose \textbf{GeoTTER}, a novel OT-based framework that addresses these challenges by systematically redesigning the OT cost matrix. Our formulation integrates prior label distribution and local geometric structure to enhance calibration. We introduce two core ideas: (1) a \textbf{graph smoothing} operation inspired by spectral graph theory to enforce local consistency, and (2) a \textbf{clustering-guided cost fusion} technique for leveraging latent structure in the embedding space. Together, these components form the backbone of GeoTTER.

\begin{figure*}
    \centering
    \includegraphics[width=1\linewidth]{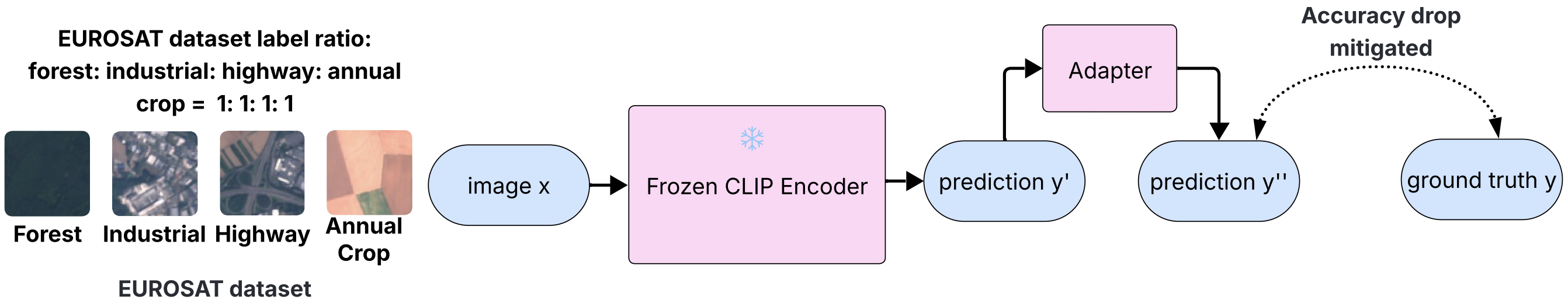}
    \caption{\textbf{Test-time adjustment (TTA) of a Frozen CLIP Encoder nder Label-Distribution Shift.} The target domain is the four-class EUROSAT subset: forest, industrial, highway, annual-crop, whose classes occur in an approximately 1: 1: 1: 1 ratio, unlike the highly skewed class prior in CLIP’s Internet-scale pre-training data (estimated: forest~$\gg$~industrial~$\approx$~highway~$\gg$~annual-crop). An unseen satellite image $x$ is first embedded by the frozen CLIP image encoder, yielding a zero-shot prediction $y'$ that is biased toward the pre-training prior and therefore suffers an accuracy drop. A lightweight adapter, learned on the unlabeled target images at test time, transforms $y'$ into an adjusted prediction $y''$, effectively compensating for the prior mismatch and recovering performance, without updating CLIP’s weights or requiring any test-set labels.}
    \label{fig:enter-label}
\end{figure*}

\textbf{Assumptions \& Scope.} The setup still requires access to the target-domain class prior \( \nu \) at test time (the paper notes that \( \nu \) can be estimated from a small labeled subset or via EM). That said, this limitation is milder than it seems: the theory frames prior misspecification as a bounded marginal-target estimation error \( \varepsilon_\nu \), and once the estimated prior \( \hat{\nu} \) achieves even modest, theoretically-guaranteed accuracy, the excess risk shrinks. Empirically, OTTER already outperforms the zero-shot (no adapter) baseline with such coarse estimation.

On top of OTTER, GeoTTER further mitigates this dependency through clustering-guided cost fusion and graph smoothing; \textit{utilizing local geometry} helps stabilize the transport plan even under imperfect prior knowledge. This local structure encourages consistency within similar samples, reducing sensitivity to \( \hat{\nu} \)'s exactness and making the estimation burden lighter in practice. Our key contributions are as follows:

$\bullet$ We introduce a graph-based smoothing technique that improves zero-shot classification accuracy by enforcing local consistency in transport assignments. \\
$\bullet$ We develop a clustering-aware extension that fuses semantic and geometric signals for improved transport planning. \\
$\bullet$ We conduct extensive experiments across 17 vision benchmarks, demonstrating that GeoTTER consistently outperforms classical OT and recent zero-shot adaptation methods \cite{shin2024ottereffortlesslabeldistribution, liusie2023mitigatingwordbiaszeroshot}.

The remainder of this paper is organized as follows. Section~\ref{sec:framework} introduces the theoretical framework of GeoTTER. Section~\ref{sec:experiments} presents empirical results, and Section~\ref{sec:conclusion} concludes with discussion and future work.

\section{RELATED WORKS}

\textbf{Optimal Transport and Cost Matrix Reformulation}
Optimal transport (OT) redistributes probability mass, but use revealed calibration drift and instability when costs span orders of magnitude. Rescaling logits with temperatures or per-class weights and embedding in the cost matrix compresses extremes and tempers over-confident predictions~\cite{sinkhorn2024}; entropy-regularized Sinkhorn stabilizes optimization with bias–variance trade-off~\cite{sinkhorn2024}. Convex re-parameterizations yield closed-form class-specific offsets, easing tuning and preserving stability in high dimensions~\cite{savare2023}.

\textbf{Graph-Based Methods and Local Manifold Regularization}
Graph methods enforce neighborhood consistency by minimizing a Laplacian quadratic on $k$-NN graphs; diffusion maps and harmonic functions extend this idea to semi-supervised learning. Advances in adaptive neighborhoods~\cite{LMRAG5}, kernel/multi-view fusion~\cite{Kernel2,Consensus3}, and large-scale solvers collectively align probability transport with intrinsic data geometry and improve robustness.

\textbf{Clustering-Guided Cost Fusion for Enhanced Structure Capture}
Beyond OT objectives, incorporating \emph{clustering cues} into cost improves calibration when labels correlate with sub-manifold structure. Pipelines discover clusters via \textit{k}-means, GMMs, or spectral variants and then fuse a cluster-level similarity term with the standard OT cost. This multi-objective view lets the OT term align global class priors while a cluster-aware term captures fine-grained geometry, yielding robust performance on imbalanced or intricately structured data.

\section{PRELIMINARIES}

\label{sec:prelim}

We start by formalizing the Test–Time Adaptation (TTA) problem tackled in
this work and by reviewing two recent baselines that recalibrate model
predictions without touching the backbone parameters.

\subsection{Setting and Notation}
\label{sec:setting}
Let \(\mathcal X=\{x_i\}_{i=1}^{n}\) be an \emph{unlabeled} test batch
and let  \(f\) be a frozen classifier which outputs softmax scores
\(
\{\mathbf p_i\}=\{\operatorname{softmax}\!\bigl(f(x_i)\bigr)\}\subset \Delta^{m-1}.
\)
The backbone model was trained on a source domain whose class prior may differ
from the \emph{target-domain prior}
\(
\boldsymbol{\nu}\in\Delta^{m-1},
\)
assumed to be known\footnote{The prior can be estimated from a small labeled subset or via expectation–maximization; empirical details are given in~\cref{sec:experiments}.}.
Our goal is to transform the scores \(\{\mathbf p_i\}\) into calibrated
probabilities \(\{\tilde{\mathbf p}_i\}\) such that their mean \(\boldsymbol{\nu}\) matches
 and the resulting hard labels are more accurate. For convenience, we write the uniform mass vector
\( \boldsymbol{\mu}=\tfrac1n\mathbf 1_n \)
and use
\(
\langle A,B\rangle
      :=\sum_{ij}A_{ij}B_{ij}
\)
for the inner products of the matrix. Below, we will briefly mention prior-matching and OTTER, two recent work tackling with recalibration of the model prediction in a zero-shot setting. 

\subsection{Prior-Matching Adapter (PM)}
\label{sec:pma}
{\bf PM}~\cite{liusie2023mitigatingwordbiaszeroshot} rescales logits with a per-class weight vector
\(
\mathbf w\in\mathbb R^{m}_{>0}
\)
and a temperature \(T>0\):
  $\tilde{\mathbf p}_i
  \;=\;
  \operatorname{softmax}\!\Bigl(
    \tfrac{1}{T}\,(\mathbf w\odot\mathbf p_i)
  \Bigr),
 $ %
where “\(\odot\)” denotes element-wise multiplication.
The weights are obtained by optimizing the following
 $ \min_{\mathbf w}\;
  \Bigl\|
    \frac1n\sum_{i=1}^{n}\tilde{\mathbf p}_i
    -\,
    \boldsymbol{\nu}
  \Bigr\|_2^{2}.
 $ %
As there are only \(m\) parameters (from the weight vector) that need to be tuned, a small number of first-order steps
suffice. This method treats each sample independently and is sensitive to the choice of \(T\)
and hyper-parameters of the optimizer. 

\textbf{Fenchel PM (FPM)} Building upon Prior Matching, we additionally evaluate Fenchel PM (FPM), which alleviates PM’s gradient-oscillation and local-minima trapping issues; FPM yields markedly better performance than PM, though it still falls short of OTTER (see Appendix for details). FPM introduce an auxiliary probability vector \( \mathbf{p} \in \Delta_m \) and, via the Fenchel--Bregman identity, absorb the softmax into an inner maximization to obtain the saddle-point problem
$
\min_{\mathbf{w} \in \mathbb{R}_{>0}^m} \; \max_{\mathbf{p} \in \Delta_m} \; \left\langle \tfrac{1}{T}(\mathbf{w} \odot \boldsymbol{\theta}), \, \mathbf{p} \right\rangle - \Omega(\mathbf{p}) + \tfrac{\lambda}{2} \| \mathbf{p} - \boldsymbol{\nu} \|_2^2,
\quad \text{where } \Omega(\mathbf{p}) = \sum_{j=1}^m p_j \log p_j.
$
The inner maximizer \( \mathbf{p}^\star(\mathbf{w}) \) satisfies the closed-form fixed point
$
\mathbf{p}^\star = \operatorname{softmax} \left( \tfrac{1}{T}(\mathbf{w} \odot \boldsymbol{\theta}) - \lambda(\mathbf{p}^\star - \boldsymbol{\nu}) \right);
$
eliminating \( \mathbf{p} \) yields a smooth, strongly convex outer objective in \( \mathbf{w} \), hence more stable optimization.

\subsection{Optimal TransporT adaptER (OTTER)}
\label{sec:otter}
OTTER~\cite{shin2024ottereffortlesslabeldistribution} interprets recalibration as a \emph{discrete optimal-transport
(OT)} problem and reallocates the entire batch’s probability mass in zeroshot. Thus it captures the relationship among samples in a batch. Below, we will explain the OT objective.

\paragraph{Transport Problem}
Given source and target marginals
\(
\boldsymbol{\mu}\text{~~and~~}\boldsymbol{\nu}
\) respectively,
we solve the \emph{exact} Kantorovich optimal-transport problem~\cite{kantorovich1942}
$ \pi^\star = \arg\min_{\pi\in\mathbb{R}_+^{n\times m}} \langle \pi, C\rangle ;\text{s.t.}; \pi,\mathbf{1}_m=\boldsymbol{\mu},; \pi^{\top}\mathbf{1}_n=\boldsymbol{\nu}. $
i.e.\ the Earth-Mover’s Distance (EMD) between the two empirical measures~\cite{rubner2000emd}.
For every sample–class pair we set
\(
  C_{ij}=-\log p_i[j].
\)
We compute \(\pi^\star\) with the \emph{network-simplex} algorithm~\cite{orlin1997networkSimplex},  
an efficient specialized linear-programming solver whose practical runtime is near \(O((n+m)^3)\) but sub-cubic on the small batches typical of OTTER. The calibrated label for each sample is then obtained as
\(
  \hat y_i = \arg\max_{j}\,\pi^\star_{ij}.
\) OTTER leverages the interaction among \emph{all} samples, dispenses with hand-tuned hyper-parameters, and unlike neural-network-based adapters, OTTER returns the exact global optimum of a well-posed linear program.

\paragraph{Exact OT / Earth-Mover’s Distance (EMD)}

Set the entropy weight to zero, i.e.\ \(\varepsilon = 0\), to recover the original Kantorovich linear program 
$
\min_{T \ge 0} \;\langle T, C \rangle
\quad\text{s.t.}\quad
T \mathbf{1} = \mathbf{a},\;
T^{\!\top}\mathbf{1} = \mathbf{b}.
$ Without the entropy term the solution \(T\) is \emph{exact} but typically sparse; standard solvers such as network–simplex or interior-point are required, yielding super-cubic time in the worst case.  

\paragraph{Sinkhorn-Regularized OTTER (SK)}

To gauge the impact of entropic regularization, we also evaluate a simple Sinkhorn-based version of OTTER (denoted \textbf{SK}), obtained by replacing the network-simplex solver with Sinkhorn iterations. Add an entropy term~\cite{sinkhorn1967diagonal}
$
\varepsilon \operatorname{KL}(T \| \mathbf{a} \mathbf{b}^\top)
= \varepsilon \sum_{ij} T_{ij} (\log T_{ij} - 1),
$
yielding
$
\min_{T \ge 0} \langle T, C \rangle + \varepsilon \sum_{ij} T_{ij} (\log T_{ij} - 1).
$ The Sinkhorn algorithm scales rows and columns alternately, giving near-linear runtime and a smoother $T$ ~\cite{cuturi2013sinkhorn}; letting $\varepsilon \to 0$ recovers EMD.

\section{THEORETICAL FRAMEWORK FOR GEOTTER}

\label{sec:framework}

Regularizing the transport plan is indispensable because it explicitly constrains the solution space, removing spurious high-frequency modes and yielding a markedly smoother assignment landscape; indeed, replacing vanilla OTTER with its regularized counterpart already delivers the consistent gains reported in Table 1. Our preliminary Sinkhorn baseline corroborates this observation: \textit{the mere addition of an entropic term systematically outperforms its unregulated analogue, suggesting that further, task-aware forms of regularization may unlock additional robustness}. Closer inspection of the error patterns reveals two complementary noise sources: a sample-level jaggedness \textbf{(lNoise)} that perturbs neighboring embeddings independently, and a batch-level drift \textbf{(gNoise)} that shifts logits coherently across larger manifolds. Addressing these dual corruptions requires two orthogonal mechanisms: (i) \textbf{Graph Smoothing} (Sec. 4.1), which averages the transport mass over the k-NN graph and therefore suppresses lNoise while respecting local geometry, and (ii) \textbf{Clustering-guided Cost Fusion} (Sec. 4.2), which re-anchors the cost matrix to cluster centroids, neutralizing the low-frequency bias induced by gNoise. Together, these strategies generalize the benefits hinted at by Sinkhorn, providing a principled, dual-scale regularization framework that underlies the full GeoTTER model. Our framework is derived from non-trivial theoretical insights in optimal transport \cite{peyre2019computational}, convex optimization \cite{boyd2004convex}, and spectral graph theory \cite{chung1997spectral}.

\subsection{Graph‑Laplacian Smoothing of the Transport Plan}
\label{sec:graph_smoothing}

\textbf{Noise Regime} The adjusted transport plan $T^{\star}$ is still vulnerable to \emph{high-frequency label jaggedness}: a sample-level stochastic logit competition in which infinitesimal perturbations, for example activation noise or batch-normalization drift, cause neighboring embeddings to exchange their top-ranked class. The resulting saw-tooth assignment pattern ~\cite{shi2023understanding, liang2022mind} violates manifold smoothness and degrades local calibration. To suppress this noise we enforce neighborhood consistency through graph-Laplacian smoothing.

\textbf{Motivation}\;
Although the adjusted cost matrix produces a globally calibrated
transport plan \(T^{\star}\), local noise in the embedding space can
still yield inconsistent assignments.  To exploit the manifold
structure and suppress this noise, we project \(T^{\star}\) onto the
$k$‑nearest‑neighbor ($k$‑NN) graph via a simple neighborhood
averaging iteration.

\paragraph{Single‑Step Propagation}
Below we give the explicit single‑step update and introduce our
notation.  Construct the $k$‑NN graph
\({\mathcal G}=(\mathcal V,\mathcal E)\) on embeddings
\(\{x_i\}_{i=1}^{n}\) with adjacency
\[
  A_{ij}
  = \mathbf 1\bigl[j\in\mathcal N(i)\bigr],\quad
  D=\mathrm{diag}(A\,\mathbf 1),\quad
  \hat A=D^{-1}A.
\]
Starting from \(T^{(0)}=T^{\star}\), we iterate
\begin{equation}
  T^{(\ell+1)}
    = \alpha\,T^{(\ell)}
    + (1-\alpha)\,\hat A\,T^{(\ell)},
  \quad 0<\alpha\le1,
  \label{eq:smoothing_update}
\end{equation}
where \(\alpha\) balances retention of the original plan against
smoothness over the $k$-NN mean ($\alpha\!\uparrow\;\Longleftrightarrow\;\text{less smoothing}$).
After \(L\) iterations we set
\(T^{\text{smooth}} = T^{(L)}\) and predict
\(\hat y_i = \arg\max_j T^{\text{smooth}}_{ij}\).
Next, we show that this simple iteration is equivalent to a graph
Laplacian regularization problem.

\noindent\textbf{(Multi-GeoTTER) Multi-Step Propagation}\\
We also provide the result of multi-step propagation, and named it as \textbf{Multi-GeoTTER}, as a GeoTTER's variant. Prop repeatedly applies the single-step update in Eq.~\eqref{eq:smoothing_update} for $L$ iterations. This yields progressively broader neighborhood smoothing while keeping the same $\alpha$. We use the final iterate $T^{(L)}$ for prediction.

\paragraph{Graph‑Laplacian Regularization}
From an optimization perspective, repeating
\eqref{eq:smoothing_update} can be seen as solving a Tikhonov‑regularized
least‑squares problem on the graph.

\begin{theorem}[Tikhonov form]
\label{lem:laplacian}
Let \(\mathcal L = I-\hat A\) be the row‑normalized graph Laplacian
and define \(\mu=\tfrac{1-\alpha}{\alpha}\).  Then the fixed point of
\eqref{eq:smoothing_update} minimizes
$
  \min_{T\in\mathbb R^{n\times m}}
    \;\frac12\sum_{i,j}\bigl(T_{ij}-T^{\star}_{ij}\bigr)^{2}
    + \mu\,\mathrm{tr}\!\bigl(T^{\!\top}\mathcal L\,T\bigr).
$
Hence, graph smoothing enforces manifold smoothness by penalizing rapid
variations of \(T\) over \(\mathcal G\).  Thus, repeatedly applying
\eqref{eq:smoothing_update} provides a simple and numerically stable
approximation to this regularized objective.
\end{theorem}

\noindent This equivalence demonstrates that the iterative update
in Eq.~\eqref{eq:smoothing_update} implicitly enforces a manifold
smoothness prior, with \(\alpha\) (or equivalently
\(\mu=\tfrac{1-\alpha}{\alpha}\)) governing the trade‑off between
fidelity to \(T^{\star}\) and local consistency.  This completes the
theoretical underpinning of our graph smoothing step.

\subsection{Clustering-Guided Cost Fusion}\label{sec:clustering-guided-cost-fusion}

\textbf{Noise Regime} Beyond local jaggedness, test batches often exhibit \emph{low-frequency coherent angular drift}: a systematic logit bias shift whereby a subset of classes receives an approximately uniform additive offset across large groups of samples, due to factors such as domain-wide prior imbalance or a persistent text--image prototype mismatch. This batch-level distortion collectively elevates the wrong labels and suppresses the true ones. We counteract the effect by fusing clustering-guided costs with the globally adjusted OT cost, thereby re-orienting the transport plan towards the correct prototypes.

In datasets with defined latent structure, further improvements can be achieved by incorporating clustering information into the OT cost. After performing unsupervised clustering (e.g., via KMeans) on the embeddings, we obtain a cluster center \(c_i\) for each sample. The clustering cost is defined based on cosine similarity, i.e., the cosine distance between the cluster center and each label encoding:
$C^{\text{cluster}}_{ij} = 1 - \text{cosine\_similarity}(c_i, y_j)$.

The final cost matrix for the \textbf{GeoTTER} fuses the adjusted OT cost \(C_{ij}\) and the clustering cost \(C^{\text{cluster}}_{ij}\) as follows:
$
C^{\mathrm{final}}_{ij}
= (1 - \beta_{\mathrm{cluster}})\, C_{ij}
+ \beta_{\mathrm{cluster}}\, C^{\mathrm{cluster}}_{ij}.
$, where \(\beta_{\text{cluster}}\) is the weighting parameter that balance the contribution of the global OT cost and the local clustering cost. \textit{Data--label cost:} $C^{\text{ot}}_{ij} = -\gamma \langle x_i, y_j \rangle $ \:\: \textit{Cluster cost:} Run $K$-means on $X$; with centre $c_{z(i)}$, define $C^{\text{cluster}}_{ij} = 1 - \cos(c_{z(i)}, y_j).$

\begin{theorem}[Cluster‑Cost Variance Bound]
\label{prop:cluster_bound}
For any sample $x_i$ in cluster $C_\ell$ with centroid $c_\ell$,
\[
  C^{\mathrm{cluster}}_{ij}
  = 1-\cos(c_\ell,y_j)
  \le
  1-\cos(x_i,y_j)
  + \frac{\|x_i-c_\ell\|_2^2}{2},
\]
by the spherical triangle inequality and cosine Taylor bound.  Averaging
over $i\in C_\ell$ yields
\[
  \mathbb{E}_{i\in C_\ell}C^{\mathrm{cluster}}_{ij}
  \le
  1-\cos(\bar x_\ell,y_j)
  + \tfrac12\,\mathrm{Var}_{i\in C_\ell}\|x_i\|_2^2,
\]
so cluster variance controls the extra cost term.

\end{theorem}

This integration results in a multi-objective optimization that simultaneously adheres to global distribution constraints and respects local geometrical structure, thereby producing a more robust and semantically coherent transport plan.

\subsection{GeoTTER when label ratio isn't available}

We extend GeoTTER with a variant that removes the requirement for target label ratios, following the approach of OTTER’s R-OTTER, which also eliminates the need for target label ratios.

We \textbf{obviates target label ratios} by estimating on $A$ a global class logit bias 
$\log\mathbf r\in\mathbb R^K$ that absorbs the prior shift implied by the GeoTTER teacher: 
with teacher $S\in[0,1]^{N_A\times K}$ constructed on $A$ via optimal transport using costs 
$C^{\mathrm{final}}_{ij}
= (1 - \beta_{\mathrm{cluster}})\, C_{ij} + \beta_{\mathrm{cluster}}\, C^{\mathrm{cluster}}_{ij} (where $$\widetilde P=\operatorname{softmax}(\widetilde L)$,  $\widetilde L=\gamma L^{(A)}+\lambda_{\mathrm{prior}}\log(\hat q+\varepsilon)$.

If enabled, 
and the OT target marginal is either $\hat q$ from EM on $A$ or uniform), 
the bias is learned by minimizing the convex objective
$
\mathcal L(\log\mathbf r)
=-\frac1{N_A}\sum_{i\in A}\sum_{k} 
S_{ik}\log\!\Bigl(\operatorname{softmax}\bigl(L^{(A)}_{i:}+\log\mathbf r\bigr)_k+\varepsilon\Bigr)
$
whose first-order condition 
$
\frac{\partial\mathcal L}{\partial(\log r_k)}
=\frac1{N_A}\sum_{i\in A}\bigl(Q_{ik}-S_{ik}\bigr)=0
$
enforces moment matching 
$\tfrac1{N_A}\sum_{i\in A}Q_{ik}=\tfrac1{N_A}\sum_{i\in A}S_{ik}$ 
with $Q_i=\operatorname{softmax}(L^{(A)}_{i:}+\log\mathbf r)$.

Since $Q_{ik}\propto e^{\log r_k}e^{L^{(A)}_{ik}}$, 
the learned $e^{\log r_k}$ acts as a multiplicative prior-odds factor 
that aligns the model’s class marginals to the teacher’s without using any information from $B$. 
Finally, at test time one simply applies the fixed bias on $B$, 
$
Q^{(B)}_{ik}=\operatorname{softmax}\bigl(L^{(B)}_{i:}+\log\mathbf r\bigr)_k,
\qquad
\hat y_i=\arg\max_k Q^{(B)}_{ik},
$
which implements the prior correction learned from $A$ and requires no access to $B$’s class proportions. The results are summarized in Table~\ref{tab:results}.

\subsection{Theoretical Evidence of Time and Memory Complexity}

\begin{table*}[t]
  \centering
  \caption{Complexity Analysis: Timing results vs. Number of samples (N)}
  \label{tab:results}
  \setlength{\tabcolsep}{1.5pt}
  \renewcommand{\arraystretch}{2}
  \small
  \begin{tabular}{lccccccccccccc}
    \toprule
    \makecell{N\\} & 128 & 256 & 512 & 1000 & 2000 & 4000 & 8000 & 12000 & 16000 & 22000 & 32000 & 50000 & 75000 \\
    \midrule
    \makecell{EMD\\(ms, MB)} 
    & \makecell{0.31\\(0.00)} & \makecell{0.45\\(0.00)} & \makecell{0.78\\(0.00)} & \makecell{1.39\\(0.00)} 
    & \makecell{4.36\\(0.00)} & \makecell{19.95\\(0.00)} & \makecell{46.00\\(0.00)} & \makecell{77.76\\(0.00)} 
    & \makecell{335.11\\(0.00)} & \makecell{214.12\\(0.00)} & \makecell{438.97\\(0.00)} & \makecell{1317.09\\(0.00)} 
    & \makecell{4945.55\\(0.00)} \\

    \makecell{Sinkhorn\\(ms, MB)} 
    & \makecell{1.89\\(82.34)} & \makecell{1.92\\(82.37)} & \makecell{1.98\\(82.44)} & \makecell{2.05\\(82.56)} 
    & \makecell{2.04\\(82.81)} & \makecell{1.90\\(83.30)} & \makecell{3.19\\(84.29)} & \makecell{2.04\\(85.29)} 
    & \makecell{2.06\\(86.28)} & \makecell{2.05\\(87.77)} & \makecell{2.42\\(225.76)} & \makecell{2.55\\(298.14)} 
    & \makecell{3.39\\(371.44)} \\

    \makecell{Total\\OTTER (ms)} 
    & \textbf{0.88} & \textbf{1.00} & \textbf{1.58} & \textbf{1.95} 
    & \textbf{4.97} & \textbf{20.60} & \textbf{47.22} & \textbf{79.13} 
    & 336.82 & 216.31 & 442.84 & 1322.84 & 4953.70 \\
    
    \makecell{Total\\GeoTTER (ms)} 
    & 203.85 & 211.25 & 209.97 & 210.00 
    & 213.76 & 216.43 & 223.97 & 247.72 
    & \textbf{254.15} & \textbf{265.45} & \textbf{289.69} & \textbf{336.15} & \textbf{381.81} \\

    \bottomrule
  \end{tabular}
\end{table*}

Let $N$ be the number of test data points, $C$ the number of classes, $D$ the embedding dimension, $I_{\text{SK}}$ the number of Sinkhorn iterations, $I_{\text{KM}}$ the FAISS ~\cite{johnson2017billionscalesimilaritysearchgpus} K-Means iterations ($\approx$ 5), and $k$ the neighbourhood size used in graph smoothing. A single forward pass of \textbf{GeoTTER} therefore consists of seven kernels:
We summarize the dominant costs of GeoTTER. Constructing the data–label cost on GPU requires $\Theta(NCD)$ time and $\Theta(NC)$ memory. The network-simplex EMD solver used by OTTER has worst-case $\Theta((N+C)^3)$ time, whereas a Sinkhorn-based solver performs two matrix–vector multiplications and a few elementwise operations per iteration, yielding $\Theta(ISK \cdot N \cdot C)$ time and $\Theta(NC)$ memory. The FAISS GPU K-means ~\cite{johnson2017billionscalesimilaritysearchgpus} with $k = C$ contributes $\Theta(IKM \cdot N \cdot C \cdot D)$ time and $\Theta(CD + N)$ memory. Computing the cluster-aware term is $\Theta(NCD)$ time / $\Theta(NC)$ memory; the elementwise cost fusion is $\Theta(NC)$. Finally, graph smoothing over the top-$k$ neighbors of the transport matrix adds $\Theta(NC \cdot \log k + N \cdot k \cdot C)$ time and $\Theta(Nk + NC)$ memory. In practice, with moderate $D$ and $k$ and small iteration counts, all GPU stages scale near-linearly in $N \cdot C$; the cubic EMD solver ~\cite{munkres1957algorithms} is the only asymptotic bottleneck, which motivates our Sinkhorn choice.

\subsection{Summary of Theoretical Contributions}

Our \textbf{GeoTTER} framework significantly advances the standard OT-based prediction methods by:

$\bullet$ \textbf{Leveraging Local Structure:} Graph smoothing, inspired by spectral graph theory, enforces local consistency in the transport plan. \\
$\bullet$ \textbf{Integrating Clustering Guidance:} The Clustering Fusion method in GeoTTER fuses global and local cost components, balancing optimal transport with latent structural information.

These theoretical developments provide a robust foundation for GeoTTER, addressing critical challenges in cost matrix formulation and optimal transport, and paving the way for improved performance in complex, real-world inference scenarios.

\section{EXPERIMENTS}
\label{sec:experiments}

In this section, we empirically evaluate the effectiveness of our proposed \textbf{GeoTTER} method across various datasets. We specifically compare its performance against several state-of-the-art methods and baseline techniques, assessing how the theoretical advancements introduced (graph smoothing, and clustering guidance) translate into improved accuracy and robustness in practice.

\subsection{Experimental Setup}
We evaluate the CLIP ViT-B/16 ~\cite{dosovitskiy2021vit} model on a range of datasets, including fine-grained object classification (Stanford-Cars, Flowers102, CUB), general object recognition (CIFAR, Caltech101, ImageNet), texture classification (DTD), and satellite imagery (EUROSAT). For a fair comparison, we follow identical preprocessing and evaluation protocols across methods and used classification accuracy as a metric of evaluation.  We also conduct the full set of experiments with additional CLIP backbones, including RN50, RN101, ViT-B/32, and ViT-L/14, to verify that our conclusions hold across architectures.

Following the evaluation practice established by TENT~\cite{wang2021tent}, SAR~\cite{niu2023sar} and the survey On Pitfalls of Test-Time Adaptation~\cite{zhao2023pitfalls}, we adopt a leave-one-dataset-out (LODO) protocol across all 17 benchmarks. In every round one dataset is withheld as the blind test set, while the remaining sixteen datasets are used to exhaustively score the full Cartesian grid of cluster-fusion (CF) and graph-smoothing (GS) hyper-parameters, 48 configurations in total, under three random seeds. The CF–GS pair that maximizes the average validation accuracy on those sixteen datasets is then frozen and applied, with no further tuning, to the held-out test set. Repeating this procedure so that each dataset serves once as the blind target guarantees that test labels are never consulted during hyper-parameter selection and that every configuration is assessed on every benchmark, yielding a rigorously fair comparison with existing test-time adaptation baselines. Additionally, Multi-GeoTTER in Table~\ref{tab:results} sweeps over the smoothing iteration amount \(L \in \{2, 3, 4, 5\}\), and selects the value of \(L\) that achieves the best performance.

\subsection{Quantitative Results}

In this section, we describe how the datasets are partitioned according to the behavior exhibited by GeoTTER when compared to traditional OT and FPM methods. Table~\ref{tab:results} provides a comprehensive accuracy comparison between GeoTTER and baseline methods: Zero-Shot (ZS), Prior Matching (PM), Fenchel-PM (FPM), and original OTTER. The highlighted results demonstrate consistent performance gains achieved by GeoTTER, clearly underscoring the effectiveness of our proposed adjustments.\textit{GeoTTER without label ratios} performs competitively with zero-shot, highlighting the effectiveness of the EM-based prior estimation. \textit{Multi-GeoTTER} uses repeated k-NN propagation to suppress high-frequency label noise and stabilize assignments, but overly aggressive smoothing blurs manifold boundaries and degrades local calibration, i.e., over-smoothing becomes a problem.

\begin{table*}[t]
  \centering
    \caption{Comparison of classification accuracy (\%) between GeoTTER variants and baselines using ViT-B/16 embeddings. Columns refer to methods defined in the paper: \textbf{ZS} (Zero-Shot baseline; Sec.~5.1), \textbf{PM} (Prior Matching; Sec.~3.2), \textbf{OTTER} (Sec.~3.3), \textbf{FPM} (Fenchel Prior Matching; Sec.~3.2 and Appendix), \textbf{Sinkhorn} (Sinkhorn-regularized OTTER; Sec.~3.3), \textbf{GeoTTER w/o Label Ratio} (Sec.~4.3), \textbf{GeoTTER w/ Label Ratio} (Secs.~4.1--4.2), and \textbf{Multi-GeoTTER} (multi-step propagation variant; Sec.~4.1). Note that PM, OTTER, FPM, Sinkhorn, GeoTTER w/ Label Ratio, and Multi-GeoTTER all uses label ratio in test time.}

  \label{tab:results}
  \setlength{\tabcolsep}{1pt}     
  \renewcommand{\arraystretch}{1.08}
  \small                          
  \begin{tabular}{lccc|cc|ccc}
    \toprule
    \multirow{}{}{Dataset} 
    & \multicolumn{3}{c|}{\textbf{Standard Baselines}} 
    & \multicolumn{2}{c|}{\textbf{Other Baselines}} 
    & \multicolumn{3}{c}{\textbf{Our Contribution}} \\
    & ZS 
    & \makecell[c]{PM\\\scriptsize \citep{liusie2023mitigatingwordbiaszeroshot}} 
    & \makecell[c]{OTTER\\\scriptsize \citep{shin2024ottereffortlesslabeldistribution}} 
    
    & FPM 
    & Sinkhorn 
    & \makecell[c]{GeoTTER\\w/o Label Ratio} 
    & \makecell[c]{GeoTTER\\w/ Label Ratio} 
    & \makecell[c]{Multi-\\ GeoTTER} \\
    \midrule

    CIFAR10            & 88.34 & 28.51 & 91.74 & 88.37 & 91.70 & 91.04 & 92.93 & \textbf{93.22} \\
    CIFAR100           & 63.82 & 60.97 & 67.94 & 65.16 & 68.12 & 63.92 & \textbf{68.88} & 67.73 \\
    Caltech101         & 79.75 & 50.66 & 88.73 & 83.57 & 57.81 & 79.55 & \textbf{89.01} & 88.31 \\
    Caltech256         & 79.80 &  8.74 & 86.99 & 82.48 & 70.81 & 78.30 & \textbf{89.62} & 89.23 \\
    Food101            & 85.59 & 79.88 & 89.87 & 86.38 & 90.01 & 85.10 & \textbf{90.46} & 89.46 \\
    STL10              & 98.03 & 19.88 & 98.61 & 98.03 & 98.69 & 98.12 & \textbf{98.95} & 98.90 \\
    SUN397             & 47.06 &  2.63 & 54.06 & 51.44 & 36.18 & 46.60 & \textbf{57.01} & 56.45 \\
    Flowers102         & 63.96 & 52.17 & 70.78 & 65.41 & 53.57 & 57.78 & 76.34 & \textbf{77.25} \\
    EUROSAT            & 32.90 & 11.36 & 42.03 & 42.76 & 42.76 & 45.12 & 47.22 & \textbf{50.91} \\
    Oxford-IIIT-Pet    & 83.84 & 29.27 & 88.83 & 83.86 & 88.66 & 79.93 & \textbf{89.53} & 88.03 \\
    Stanford-Cars      & 55.70 & 32.02 & 59.71 & 54.31 & 59.45 & 49.44 & \textbf{62.83} & 61.71 \\
    Country211         & 19.83 & 18.91 & 21.09 & 21.16 & 21.27 & 18.76 & \textbf{21.14} & 19.34 \\
    DTD                & 39.04 & 13.46 & 44.41 & 39.57 & 44.95 & 41.62 & \textbf{48.30} & 47.77 \\

    CUB                & 45.98 & 40.33 & 50.40 & 47.79 & 51.00 & 40.99 & \textbf{55.37} & 54.38 \\
    ImageNet           & 60.18 & 49.84 & 62.86 & 61.09 & 63.14 & 53.51 & \textbf{65.93} & 63.83 \\
    ImageNet-R         & 68.85 & 12.08 & 72.40 & 70.42 & 60.72 & 64.85 & \textbf{73.58} & 69.91 \\
    ImageNet-Sketch    & 39.79 & 30.68 & 44.48 & 42.97 & 44.89 & 36.56 & \textbf{46.61} & 45.13 \\
    \bottomrule
  \end{tabular}
\end{table*}

\subsection{Hyperparameter Ablation Studies}
In this section, we explain ablation studies of various hyper-parameters used. Additional results and details are provided in Appendix.

\textbf{Ablation Setup} One of the experiments uses EuroSAT (22,000 images, 10 classes) and a frozen CLIP ViT-B/16 encoder. GeoTTER is always run a single step Laplacian smoothing; we vary only four factors: (i) whether Sinkhorn transport is applied; (ii) neighborhood size $k \in \{1,\,2,\,3,\,4,\,5,\,7,\,10,\,20,\,30,\,50\}$; (iii) smoothing coefficient $\alpha \in \{0,\,0.1,\,0.2,\,0.3,\,0.4,\,0.5,\,0.7,\,1\}$; and (iv) clustering-fusion ($\beta \in \{0,\,0.3,\,0.5,\,0.6,$ $\,0.7,\,0.8,\,0.9,\,1\}$).

\vspace{-0.5em}
\subsection{Hyperparameter Sensitivity}%
\vspace{-0.75em}
Under a simple variance model, the excess risk decomposes as
$
\mathcal{E}(\beta,\alpha)
=\beta^{2}\tau_{\ell}^{2}+(1-\beta)^{2}\sigma^{2}
+\eta^{2}\!\Bigl[\alpha^{2}+\tfrac{(1-\alpha)^{2}}{k}\Bigr],
$
whose unique minimizers are
$
\beta^{\star}_{\ell}=\frac{\sigma^{2}}{\sigma^{2}+\tau_{\ell}^{2}}
\quad\text{and}\quad
\alpha^{\star}=\frac{1}{1+k}.
$
Around this point, errors grow \emph{quadratically} and the mixed term is dominated—there is no harmful interaction between $\alpha$ and $\beta$. This justifies using the closed-form defaults without grid search and explains the observed robustness to $\pm0.2$ perturbations. Experimentally, across wide sweeps of $(\beta,k,\alpha)$, GeoTTER maintains a uniform margin over OTTER (see Appendix for details); a simple default of parameters stable accuracy attains datasets/backbones, while all other hyperparameter slices still exceed OTTER by large absolute margins, mirroring the quadratic stability and indicating practical insensitivity to hyperparameter choice.
\vspace{-1em}
\subsection{Noise-Reduction Experiment}
\vspace{-0.75em}
This section evaluates how well \textbf{Cluster Fusion (CF)} and \textbf{Graph Smoothing (GS)} suppress the two variance sources they were designed for. We first introduce a controlled perturbation model, then state two theorems that quantify the variance that survives CF and GS, and finally describe the empirical protocol that confirms the theory. Setting $\sigma^{2}=0$ or $\eta=0$ isolates either perturbation; choosing non-zero values for both yields a mixed attack. Full derivations are relegated to Appendix.

\vspace{-0.5em}
\textbf{Perturbation model} Let $\mathbf x_i\!\in\!\mathbb R^{D}$ be an $\ell^{2}$-normalized embedding and let $\mathbf p_c$ be the prototype of class $c$. Two independent noise mechanisms perturb each sample as follows: 
(1) \emph{gNoise}: $\tilde{\mathbf x}_i = \mathbf x_i + \sqrt{\sigma^{2}}\;\mathbf z_i$ with $\mathbf z_i \sim \mathcal N(\mathbf 0,\mathbf I_D)$, which independently adds isotropic Gaussian noise to each dimension, inflating sample-level variance while leaving cluster centroids almost unchanged; 
(2) \emph{lNoise}: $\tilde{\mathbf x}_i = \mathbf x_i + \eta\,\frac{\mathbf u_i - (\mathbf u_i^{\top}\mu_i)\mu_i}{\lVert \mathbf u_i - (\mathbf u_i^{\top}\mu_i)\mu_i\rVert}$ where $\mathbf u_i \sim \mathcal N(\mathbf 0,\mathbf I_D)$ and $\mu_i = \frac1k\sum_{j\in\mathcal N_k(i)} \mathbf x_j$, which pushes each sample a short distance along a random vector orthogonal to the subspace of itself and its $k$ nearest neighbors, preserving radial projections to distant prototypes but disrupting local cosine coherence.

\textbf{Theoretical guarantees (gNoise):} Writing $C^{\mathrm{OT}}_{ic}=-\log\bigl(\operatorname{softmax}(\gamma\langle\tilde{\mathbf x}_i,\mathbf p_c\rangle/T_0)_c\bigr)$ with softmax probability $P_{ic}$, and letting $\tau_\ell^{2}$ denote the across-prototype variance of the clean centroid of K-Means cluster $\mathcal C_\ell$, we have the following. Here, $\Delta C^{\mathrm{fin}}_{ic}:=C^{\mathrm{fin}}_{ic}(\tilde{\mathbf x}_i)-C^{\mathrm{fin}}_{ic}(\mathbf x_i)$ denotes the perturbation in the final blended cost for sample $i$ and prototype $c$, and $n_\ell:=|\mathcal C_\ell|$ is the cluster size. If embeddings are perturbed \emph{only} by gNoise of variance $\sigma^{2}$ and the blended cost uses the shrinkage weight $\beta_{\mathrm{cluster}}^{\star}= \sigma^{2}/(\sigma^{2}+\tau_\ell^{2})$, then for every sample $i$ in cluster $\ell$ and every prototype $c$ the additional variance in the final cost satisfies
\[
\operatorname{Var}\!\bigl[\Delta C_{ic}^{\mathrm{fin}}\bigr]
  =\bigl((1-P_{ic})\tfrac{\gamma}{T_0}\bigr)^2
    \frac{\sigma^{2}\tau_\ell^{2}}{\sigma^{2}+\tau_\ell^{2}}
  <\bigl((1-P_{ic})\tfrac{\gamma}{T_0}\bigr)^2\sigma^{2},
\]
so CF attenuates gNoise by the factor $\tau_\ell^{2}/(\sigma^{2}+\tau_\ell^{2})<1$.  A first-order expansion shows that gNoise adds an independent Gaussian error of variance $\sigma^{2}$ to each logit; averaging $n_\ell$ such errors inside the cluster divides the variance by $n_\ell$, and the James–Stein coefficient $\beta_{\mathrm{CL}}^{\star}$ minimizes the mean-squared error, yielding the stated reduction (full derivation in Appendix).

\textbf{Theoretical guarantees (lNoise)} If embeddings are perturbed \emph{only} by lNoise of magnitude $\eta$, the $k$-NN graph is built on those embeddings, and graph smoothing uses neighborhood size $k\!\ge\!2$ with weight $0<\alpha<1$, then the temperature estimate obeys
\[
\operatorname{Var}\!\bigl[\tilde T_{ic}-T^{\mathrm{clean}}_{ic}\bigr]
  =\eta^{2}\!\left[\alpha^{2}+\frac{(1-\alpha)^{2}}{k}\right]
  <\eta^{2},
\]
implying that GS multiplies lNoise variance by $\alpha^{2}+(1-\alpha)^{2}/k<1$.  Because lNoise is orthogonal to the local subspace, it contributes an independent zero-mean perturbation of variance $\eta^{2}$; graph smoothing retains an $\alpha$ fraction of this noise and averages the remainder over $k$ neighbors, giving the stated bound. Here, $\tilde T_{ic}$ and $T^{\mathrm{clean}}_{ic}$ denote the GS-smoothed and clean OT-transport plans for sample $i$ and class $c$; concretely, $\tilde T_{ic}=\alpha T^\star_{ic}+(1-\alpha)\,k^{-1}\!\sum_{j\in\mathcal N_k(i)}T^\star_{jc}$, where $\mathcal N_k(i)$ is the $k$-NN set on the perturbed embeddings and $T^\star$ is the pre-smoothing OT row.

\textbf{Noise-Reduction Experiment Results} To evaluate how well our geometric refinements mitigate label noise, we inject label noise to test noise robustness. The key findings are summarized below, and additional results are provided in the appendix.

$\bullet$ \textbf{Evidence for gNoise and lNoise}
Even without injected noise, GeoTTER consistently improves over OTTER, indicating that both batch-level coherent drift (gNoise) and sample-level label jaggedness (lNoise) are present in real datasets. Cluster-guided cost fusion and graph-Laplacian smoothing respectively target these effects, and their gains persist across all 17 benchmarks; see Appendix figures for the full ablation.
$\bullet$ \textbf{Robustness Under Increasing Noise Levels} When the noise level is low, GeoTTER already slightly outperforms OTTER, indicating that the additional geometric processing does not harm performance in clean settings. As the noise increases, the performance gap between GeoTTER and OTTER gradually widens, suggesting that the proposed geometric refinements effectively suppress the impact of noise. However, when the noise becomes excessively large, the performance of both GeoTTER and OTTER rapidly degrades and converges toward the zeroshot baseline. This indicates that, beyond a certain threshold, the signal-to-noise ratio becomes too low for prototype-based corrections to recover meaningful classification signals. 

\vspace*{-0.75em}

\section{CONCLUSION}

\label{sec:conclusion}
\vspace*{-0.75em}

This work introduced GeoTTER, a novel framework that incorporate graph smoothing, grounded in spectral graph theory, ensures consistency across neighboring representations, while the clustering-guided cost fusion in the GeoTTER successfully captures latent structural information within the data. Our experimental results, obtained using a ViT-B/16 embedding backbone across seventeen diverse datasets, indicate that GeoTTER not only outperforms existing baseline methods but also offers significant improvements in challenging scenarios such as fine-grained classification and noisy data environments.

\textbf{Future Work}
Looking forward, we highlight two potential research directions below for extending GeoTTER. We believe that the insights and innovations presented in this work lay a solid foundation for future advancements in OT-based methods and their applications in complex, real-world inference problems.

$\bullet$ \textbf{Streaming and Online Test-Time Adaptation}
GeoTTER naturally extends to streaming scenarios by (i) maintaining class-prior estimates via sliding windows or exponential moving averages (optionally using soft labels for Bayesian updates), (ii) solving local OT on incoming mini-batches to produce soft labels that are merged into a running summary, and (iii) incrementally fusing transport plans via online OT or barycenters with warm-started dual variables. With budgeted memory, one may store O(W + C) vectors or a small set of representative prototypes (e.g., online k-means). These ingredients retain GeoTTER’s benefits under compute-constrained online settings.

$\bullet$ \textbf{Compatibility with Prompt Ensembling}
GeoTTER is complementary to prompt-level techniques such as prompt ensembling (e.g., DCLIP). Whereas prompt ensembling strengthens prototype quality, GeoTTER calibrates predictions at the label level by aligning test-time class priors and local geometry; these factors are orthogonal. Consequently, GeoTTER can be layered on top of prompt ensembles to add noise-aware calibration without any additional prompt engineering

\vspace{-0.5em}
\textbf{Acknowledgments.} Prof. Lokhande acknowledges support from University at Buffalo startup funds, an Adobe Research Gift, an NVIDIA Academic Grant, and the National Center for Advancing Translational Sciences of the NIH (award UM1TR005296 to the University at Buffalo). Dr.\ Chakraborty performed this work under the auspices of the U.S.\ Department of Energy by Lawrence Livermore National Laboratory under Contract DE-AC52-07NA27344.

\bibliographystyle{ACM-Reference-Format}
\bibliography{reference}

\begin{thebibliography}{}

\bibitem[Bai et~al., 2023]{bai2023adaptingonlinelabelshift}
Bai, Y., Zhang, Y.-J., Zhao, P., Sugiyama, M., and Zhou, Z.-H. (2023).
\newblock Adapting to online label shift with provable guarantees.

\bibitem[Boyd and Vandenberghe, 2004a]{boyd2004convex}
Boyd, S. and Vandenberghe, L. (2004a).
\newblock {\em Convex Optimization}.
\newblock Cambridge University Press, Cambridge, UK.

\bibitem[Boyd and Vandenberghe, 2004b]{boyd2004}
Boyd, S. and Vandenberghe, L. (2004b).
\newblock {\em Convex Optimization}.
\newblock Cambridge University Press.

\bibitem[Bubeck, 2015]{bubeck2015}
Bubeck, S. (2015).
\newblock {\em Convex Optimization: Algorithms and Complexity}, volume~8.
\newblock Foundations and Trends in Machine Learning.

\bibitem[Chen et~al., 2022]{Kernel2}
Chen, L., Xu, K., and Zhao, P. (2022).
\newblock Non-sparse kernel integration for robust representation learning.
\newblock {\em IEEE Transactions on Neural Networks and Learning Systems}, 33(5):2401--2413.

\bibitem[Chung, 1997]{chung1997spectral}
Chung, F. R.~K. (1997).
\newblock {\em Spectral Graph Theory}, volume~92 of {\em CBMS Regional Conference Series in Mathematics}.
\newblock American Mathematical Society, Providence, RI.

\bibitem[Cuturi, 2013a]{cuturi2013sinkhorn}
Cuturi, M. (2013a).
\newblock Sinkhorn distances: Lightspeed computation of optimal transport.
\newblock In {\em Advances in Neural Information Processing Systems}, volume~26, pages 2292--2300.

\bibitem[Cuturi, 2013b]{cuturi2013sinkhorndistanceslightspeedcomputation}
Cuturi, M. (2013b).
\newblock Sinkhorn distances: Lightspeed computation of optimal transportation distances.

\bibitem[Devlin et~al., 2018]{Devlin2018bert}
Devlin, J., Chang, M., Lee, K., and Toutanova, K. (2018).
\newblock {BERT}: Pre-training of deep bidirectional transformers for language understanding.
\newblock {\em CoRR}, abs/1810.04805.

\bibitem[Dosovitskiy et~al., 2020]{Dosovitskiy2020vit}
Dosovitskiy, A., Beyer, L., Kolesnikov, A., Weissenborn, D., Zhai, X., Unterthiner, T., Dehghani, M., Minderer, M., Heigold, G., Gelly, S., Uszkoreit, J., and Houlsby, N. (2020).
\newblock An image is worth 16x16 words: Transformers for image recognition at scale.
\newblock {\em CoRR}, abs/2010.11929.

\bibitem[Dosovitskiy et~al., 2021]{dosovitskiy2021vit}
Dosovitskiy, A., Beyer, L., Kolesnikov, A., Weissenborn, D., Zhai, X., Unterthiner, T., Dehghani, M., Minderer, M., Heigold, G., Gelly, S., Uszkoreit, J., and Houlsby, N. (2021).
\newblock An image is worth 16×16 words: Transformers for image recognition at scale.
\newblock In {\em Proceedings of the 9th International Conference on Learning Representations (ICLR)}.
\newblock arXiv:2010.11929.

\bibitem[Fawzi et~al., 2018]{fawzi2015}
Fawzi, A., Fawzi, O., and Frossard, P. (2018).
\newblock Analysis of classifiers' robustness to adversarial perturbations.
\newblock {\em Machine Learning}, 107(3):481--508.
\newblock arXiv:1502.02590.

\bibitem[He et~al., 2015]{He2015resnet}
He, K., Zhang, X., Ren, S., and Sun, J. (2015).
\newblock Deep residual learning for image recognition.
\newblock {\em CoRR}, abs/1512.03385.

\bibitem[James and Stein, 1961]{james_stein_1961}
James, W. and Stein, C. (1961).
\newblock Estimation with quadratic loss.
\newblock In {\em Proceedings of the Fourth Berkeley Symposium on Mathematical Statistics and Probability, Volume 1: Contributions to the Theory of Statistics}, pages 361--379, Berkeley, CA. University of California Press.

\bibitem[Johnson et~al., 2017]{johnson2017billionscalesimilaritysearchgpus}
Johnson, J., Douze, M., and Jégou, H. (2017).
\newblock Billion-scale similarity search with gpus.

\bibitem[Kantorovich, 1942]{kantorovich1942}
Kantorovich, L.~V. (1942).
\newblock On the translocation of masses.
\newblock {\em Doklady Akademii Nauk SSSR}, 37(7--8):199--201.
\newblock English translation: Management Science, 5(1):1--4, 1958.

\bibitem[Karimi et~al., 2016]{karimi2016}
Karimi, H., Nutini, J., and Schmidt, M. (2016).
\newblock Linear convergence of gradient and proximal-gradient methods under the polyak-{\l}ojasiewicz condition.
\newblock In {\em Machine Learning and Knowledge Discovery in Databases (ECML PKDD)}, pages 795--811. Springer.

\bibitem[Lavenant et~al., 2024]{sinkhorn2024}
Lavenant, H., Luckhardt, J., Mordant, G., Schmitzer, B., and Tamanini, L. (2024).
\newblock The riemannian geometry of sinkhorn divergences.

\bibitem[Liang et~al., 2022]{liang2022mind}
Liang, V.~W., Zhang, Y., Kwon, Y., Yeung, S., and Zou, J.~Y. (2022).
\newblock Mind the gap: Understanding the modality gap in multi-modal contrastive representation learning.
\newblock {\em Advances in Neural Information Processing Systems}, 35:17612--17625.

\bibitem[Liusie et~al., 2023]{liusie2023mitigatingwordbiaszeroshot}
Liusie, A., Manakul, P., and Gales, M. J.~F. (2023).
\newblock Mitigating word bias in zero-shot prompt-based classifiers.

\bibitem[Munkres, 1957]{munkres1957algorithms}
Munkres, J. (1957).
\newblock Algorithms for the assignment and transportation problems.
\newblock {\em Journal of the society for industrial and applied mathematics}, 5(1):32--38.

\bibitem[Nesterov, 2004]{nesterov2004}
Nesterov, Y. (2004).
\newblock {\em Introductory Lectures on Convex Optimization: A Basic Course}, volume~87 of {\em Applied Optimization}.
\newblock Springer.

\bibitem[Niu et~al., 2023]{niu2023sar}
Niu, S., Wu, J., Zhang, Y., Wen, Z., Chen, Y., Zhao, P., and Tan, M. (2023).
\newblock Towards stable test-time adaptation in dynamic wild world.
\newblock In {\em International Conference on Learning Representations (ICLR)}.
\newblock Oral.

\bibitem[{OpenAI}, 2022]{OpenAI2022ada002}
{OpenAI} (2022).
\newblock New and improved embedding model: \texttt{text-embedding-ada-002}.
\newblock \url{https://openai.com/index/new-and-improved-embedding-model/}.
\newblock Accessed 23 May 2025.

\bibitem[Orlin, 1997]{orlin1997networkSimplex}
Orlin, J.~B. (1997).
\newblock A polynomial time primal network simplex algorithm for minimum cost flows.
\newblock {\em Mathematical Programming}, 78(2):109--129.

\bibitem[Peyr{\'e} and Cuturi, 2019]{peyre2019computational}
Peyr{\'e}, G. and Cuturi, M. (2019).
\newblock Computational optimal transport.
\newblock {\em Foundations and Trends in Machine Learning}, 11(5--6):355--607.

\bibitem[Rubner et~al., 2000]{rubner2000emd}
Rubner, Y., Tomasi, C., and Guibas, L.~J. (2000).
\newblock The earth mover's distance as a metric for image retrieval.
\newblock {\em International Journal of Computer Vision}, 40(2):99--121.

\bibitem[Savaré and Sodini, 2023]{savare2023}
Savaré, G. and Sodini, L. (2023).
\newblock Convex relaxation of unbalanced ot via thermodynamic principles.
\newblock {\em Foundations of Computational Mathematics}, 23(2):567--598.

\bibitem[Shi et~al., 2023]{shi2023understanding}
Shi, P., Welle, M., Bj{\o}rkman, M., and Kragic, D. (2023).
\newblock Understanding the modality gap in clip.
\newblock {\em ICLR, Stockholm, Sweden}.

\bibitem[Shin et~al., 2024]{shin2024ottereffortlesslabeldistribution}
Shin, C., Zhao, J., Cromp, S., Vishwakarma, H., and Sala, F. (2024).
\newblock Otter: Effortless label distribution adaptation of zero-shot models.

\bibitem[Sinkhorn and Knopp, 1967]{sinkhorn1967diagonal}
Sinkhorn, R. and Knopp, P. (1967).
\newblock Concerning nonnegative matrices and doubly stochastic matrices.
\newblock {\em Pacific Journal of Mathematics}, 21(2):343--348.

\bibitem[Stein, 1956]{stein_1956_inadmissible}
Stein, C. (1956).
\newblock Inadmissibility of the usual estimator for the mean of a multivariate normal distribution.
\newblock In {\em Proceedings of the Third Berkeley Symposium on Mathematical Statistics and Probability, Volume 1: Contributions to the Theory of Statistics}, pages 197--206, Berkeley, CA. University of California Press.

\bibitem[Vershynin, 2018]{vershynin2018}
Vershynin, R. (2018).
\newblock {\em High-Dimensional Probability: An Introduction with Applications in Data Science}.
\newblock Cambridge University Press.

\bibitem[Wang et~al., 2021a]{wang2021tent}
Wang, D., Shelhamer, E., Liu, S., Olshausen, B., and Darrell, T. (2021a).
\newblock {Tent}: Fully test-time adaptation by entropy minimization.
\newblock In {\em International Conference on Learning Representations (ICLR)}.

\bibitem[Wang et~al., 2021b]{Consensus3}
Wang, Y., Jin, F., and Liu, H. (2021b).
\newblock Multi-view consensus learning with local structure preservation.
\newblock In {\em Proceedings of the International Conference on Machine Learning (ICML)}, pages 5678--5687.

\bibitem[Zhang et~al., 2020]{LMRAG5}
Zhang, W., Li, J., and Wang, R. (2020).
\newblock Learning manifold-regularized adaptive graphs for semi-supervised classification.
\newblock In {\em Proceedings of the IEEE Conference on Computer Vision and Pattern Recognition (CVPR)}, pages 1234--1243.

\bibitem[Zhao et~al., 2023]{zhao2023pitfalls}
Zhao, H., Liu, Y., Alahi, A., and Lin, T. (2023).
\newblock On pitfalls of test-time adaptation.
\newblock In {\em Proceedings of the 40th International Conference on Machine Learning (ICML)}, volume 202 of {\em Proceedings of Machine Learning Research}, pages 42058--42080.

\bibitem[Zhu et~al., 2003]{zhu2003}
Zhu, X., Ghahramani, Z., and Lafferty, J.~D. (2003).
\newblock Semi-supervised learning using gaussian fields and harmonic functions.
\newblock In {\em Proc. ICML 20}, pages 912--919.

\end{thebibliography}

\section*{Checklist}

\begin{enumerate}

  \item For all models and algorithms presented, check if you include:
  \begin{enumerate}
    \item A clear description of the mathematical setting, assumptions, algorithm, and/or model. [Yes]
    \item An analysis of the properties and complexity (time, space, sample size) of any algorithm. [Yes]
    \item (Optional) Anonymized source code, with specification of all dependencies, including external libraries. [No]
  \end{enumerate}

  \item For any theoretical claim, check if you include:
  \begin{enumerate}
    \item Statements of the full set of assumptions of all theoretical results. [Yes]
    \item Complete proofs of all theoretical results. [No]
    \item Clear explanations of any assumptions. [Yes]     
  \end{enumerate}

  \item For all figures and tables that present empirical results, check if you include:
  \begin{enumerate}
    \item The code, data, and instructions needed to reproduce the main experimental results (either in the supplemental material or as a URL). [No]
    \item All the training details (e.g., data splits, hyperparameters, how they were chosen). [Yes]
    \item A clear definition of the specific measure or statistics and error bars (e.g., with respect to the random seed after running experiments multiple times). [Yes]
    \item A description of the computing infrastructure used. (e.g., type of GPUs, internal cluster, or cloud provider). [Yes]
  \end{enumerate}

  \item If you are using existing assets (e.g., code, data, models) or curating/releasing new assets, check if you include:
  \begin{enumerate}
    \item Citations of the creator If your work uses existing assets. [Yes]
    \item The license information of the assets, if applicable. [Not Applicable]
    \item New assets either in the supplemental material or as a URL, if applicable. [Not Applicable]
    \item Information about consent from data providers/curators. [Not Applicable]
    \item Discussion of sensible content if applicable, e.g., personally identifiable information or offensive content. [Not Applicable]
  \end{enumerate}

  \item If you used crowdsourcing or conducted research with human subjects, check if you include:
  \begin{enumerate}
    \item The full text of instructions given to participants and screenshots. [Not Applicable]
    \item Descriptions of potential participant risks, with links to Institutional Review Board (IRB) approvals if applicable. [Not Applicable]
    \item The estimated hourly wage paid to participants and the total amount spent on participant compensation. [Not Applicable]
  \end{enumerate}

\end{enumerate}

\onecolumn

\aistatstitle{GeoTTER: Leveraging Local Geometry of Optimal Transport for Zero-Shot Classification: Supplementary Materials}

\section*{Appendix Outline}

\begin{description}
  \item[Appendix~A] Summary of Symbols used in the paper \dotfill\pageref{app:SummaryofSymbols}

  \item[Appendix~B] Mathematical Models and Proofs                                 \dotfill\pageref{app:proofs}
    \begin{description}
      \item[Appendix~B.1] Proof that Clustering Fusion can reduce Label Noise                              \dotfill\pageref{app:pfCF}
      \item[Appendix~B.2] Upper Bounding gNoise and lNoise with GeoTTER        \dotfill\pageref{app:boundingNoises}

      \item[Appendix~B.3] Proof of Misclassification Margin in Long-Tail Data                      \dotfill\pageref{app:margin}
    \end{description}

  \item[Appendix~C] More details on Experiments and additional results \dotfill\pageref{app:exp}
    \begin{description}
      \item[Appendix~C.1] Experiments with different backbones        \dotfill\pageref{app:mainResults}
      \item[Appendix~C.1.1] ViT \cite{Dosovitskiy2020vit} backbones (ViT-B/32, ViT-L/14)           \dotfill\pageref{app:vitseries}
      \item[Appendix~C.1.2] ResNet \cite{He2015resnet} backbones (RN50, RN101)                       \dotfill\pageref{app:rn50101}
      \item[Appendix~C.1.2] BERT \cite{Devlin2018bert} backbone and ADA \cite{OpenAI2022ada002} backbone                       \dotfill\pageref{app:bertada}
      \item[Appendix~C.2] More Experiments on Clustering Fusion and Graph Smoothing        \dotfill\pageref{app:additionalexp}

      \item[Appendix~C.2.1] Benefits of Cluster Fusion       \dotfill\pageref{app:benefit}
      \item[Appendix~C.2.2] Changes of Parameters in Clustering Fusion and Graph Smoothing \dotfill\pageref{app:changeParam}

      \item[Appendix~C.2.3] Change of Iterations in Prop (the Propagation-GeoTTER variant)                  \dotfill\pageref{app:changeProp}

      \item[Appendix~C.3] Experiment on gNoise and lNoise    \dotfill\pageref{app:C3}
      \item[Appendix~C.3.1] GeoTTER can reduce gNoise and lNoise \dotfill\pageref{app:reduce}
      \item[Appendix~C.3.2] Impact of gNoise and lNoise on Synthetic Long-Tail Data                       \dotfill\pageref{app:Synthetic}
          \end{description}
        \item[Appendix~D]  Implementation Details for Toy Example Figures \dotfill\pageref{app:Toy}

  \item[Appendix~E] Statistics of datasets \dotfill\pageref{app:Statistics}

  \item[Appendix~F] Additional related work           \dotfill\pageref{app:Related}
  \begin{description}
  \item[Appendix~F.1] Fenchel PM (FPM)                    \dotfill\pageref{app:Fenchel}
    \end{description}
\end{description}

\newpage
\section*{Appendix Overview}
This appendix contains a summary of symbols, extended related discussion, method details, formal proofs, and comprehensive experimental results.

\section{Summary of Symbols used in the paper. }\label{app:SummaryofSymbols}

\begin{table}[H]
\centering
\begin{tabular}{ll}
\toprule
\textbf{Symbol / Notation} & \textbf{Meaning in This Paper} \\
\midrule
$n$ & Size of the test batch $X$; the number of unlabeled samples $x_i$. \\
$m$ & Number of classes (dimension of every probability vector). \\
$X = \{x_i\}_{i=1}^n$ & Unlabeled test-time batch in the target domain. \\
$f$ & Frozen backbone classifier that outputs logits for each $x_i$. \\
$p_i=\mathrm{softmax}\!\bigl(f(x_i)\bigr)$ & Raw soft-max scores for $x_i$ in probability simplex $\Delta_{m-1}$. \\
$\tilde{p}_i$ & Calibrated probability vector obtained after Test-Time Adaptation (TTA). \\
$\nu \in \Delta_{m-1}$ & \textbf{Target-domain class prior} (assumed known or estimated). \\
$\mu = \tfrac1n \mathbf 1_n$ & Uniform mass vector used for optimal-transport constraints. \\
$w \in \mathbb R_{>0}^{m}$ & Per-class scaling weights learned by Prior-Matching (PM). \\
$T > 0$ & Temperature parameter that jointly rescales all logits in PM/FPM. \\
$\odot$ & Element-wise (Hadamard) product of two vectors. \\
$\langle A, B\rangle = \sum_{i,j} A_{ij} B_{ij}$ & Frobenius inner product between matrices $A$ and $B$. \\
$C \in \mathbb R^{n\times m}$ & OT cost matrix with entries $C_{ij} = -\log p_i[j]$. \\
$\pi^{\star}$ & Optimal transport plan solving the Kantorovich problem in Eq. (1). \\
$\varepsilon$ & Entropy weight used in the Sinkhorn-regularized OT objective. \\
$T$ (in OT) & Generic transport matrix variable (distinct from temperature $T$). \\
$A$ & Adjacency matrix of the $k$-nearest-neighbor graph on embeddings. \\
$D = \mathrm{diag}(A\mathbf 1)$ & Degree matrix corresponding to $A$. \\
$\hat A = D^{-1}A$ & Row-normalized adjacency used for graph smoothing. \\
$L = I - \hat A$ & Row-normalized graph Laplacian. \\
$\alpha\in(0,1]$ & Mixing coefficient that balances raw and smoothed transport mass (Eq. 2). \\

$c_{z(i)}$ & Centroid of the cluster containing sample $x_i$ (from K-Means). \\
$C_{ij}^{\text{cluster}} = 1 - \cos\bigl(c_{z(i)}, y_j\bigr)$ & Cluster-level cost term used in clustering-guided cost fusion. \\
$\beta_{\text{cluster}}$ & Weight assigned to the clustering cost in GeoTTER’s fused cost matrix. \\
\bottomrule
\end{tabular}
\caption{Glossary of symbols and definitions used in Sections 3--4.}
\end{table}

\section{Mathematical Models and Proofs}\label{app:proofs}

\subsection{Proof that Clustering Fusion can reduce Label Noise}\label{app:pfCF}
To complete the argument sketched in §4.2, we must show that \textbf{the clustering‐guided cosine cost inherits a James--Stein--style variance-shrinkage on the unit sphere and that this shrinkage propagates to an explicit upper bound on every entry of the fused cost matrix}. Concretely, for any sample \(x_i\) inside a cluster \(\mathcal C_\ell\) with centroid \(c_\ell\),

\begin{itemize}
\item The slerp-shrunken vector  
  \(\displaystyle\widetilde x_i\bigl(\beta_\ell^{\star}\bigr)\) with  
  \(\beta_\ell^{\star}=\dfrac{\sigma^{2}}{\sigma^{2}+\tau_\ell^{2}}\)  
  minimizes the spherical quadratic risk  
  \(\mathbb E\!\bigl[1-\cos(\hat u,\hat s_\ell)\bigr]\);
\item The corresponding cosine distance to any label prototype \(y_j\) satisfies  
\[
d_{\cos}\!\bigl(\widetilde x_i(\beta_\ell^{\star}),y_j\bigr)
   \le
   (1-\beta_\ell^{\star})\,d_{\cos}(x_i,y_j)
   +\beta_\ell^{\star}\,d_{\cos}(c_\ell,y_j)
   +\tfrac12\lVert\hat x_i-\hat c_\ell\rVert^{2},
\]
\end{itemize}

so that \textbf{cluster variance alone controls the extra term}.  
These two facts together justify the ``cluster-cost variance bound’’ quoted in Theorem 4.2 and guarantee that the OT plan produced by the fused cost \(C_{ij}^{\text{final}}\) remains both globally faithful and locally robust.

\bigskip

\noindent\textbf{Lemma C.1 (Tangent-plane risk identity).}  
With the unified notation  
\[
x_i=s_\ell+\varepsilon_i,\quad
\varepsilon_i\sim\mathcal N(\mathbf0,\sigma^{2}\mathbf I_d),\qquad
c_\ell=\frac1{n_\ell}\sum_{k\in\mathcal C_\ell}x_k
      =s_\ell+\bar\varepsilon_\ell ,
\]
write the within-cluster spread  
\[
\tau_\ell^{2}:=\frac1{n_\ell}\sum_{k\in\mathcal C_\ell}
               \lVert x_k-s_\ell\rVert^{2},
\qquad
\operatorname{Var}[\bar\varepsilon_\ell]=\sigma^{2}/n_\ell.
\]
Define the \emph{spherical} interpolation (slerp)  
\[
\widetilde x_i(\beta)=
  \frac{\sin((1-\beta)\theta)}{\sin\theta}\,\hat x_i+
  \frac{\sin(\beta\theta)}{\sin\theta}\,\hat c_\ell,
\quad
\theta:=\arccos\langle\hat x_i,\hat c_\ell\rangle,\;
0\le\beta\le1.
\]
For the cosine loss  
\[
L(u):=1-\cos\bigl(\hat u,\hat s_\ell\bigr)
     =\tfrac12\lVert\hat u-\hat s_\ell\rVert^{2},
\]
the risk expands to  
\[
\boxed{
\mathbb E\bigl[L(\widetilde x_i(\beta))\bigr]
      =(1-\beta)^{2}\sigma^{2}+\beta^{2}\tau_\ell^{2}
      +o(\sigma^{2})
}\quad(\star)
\]

\noindent\emph{(no terms are hidden: the remainder is a curvature-only
\(o(\sigma^{2})\) whose explicit bound is given below).}

\medskip

\noindent\textit{Proof.}  
\begin{enumerate}
\item \textbf{Tangent-plane decomposition.}  
   In the tangent space \(T_{\hat s_\ell}\mathbb S^{d-1}\),
   \[
     \widetilde x_i(\beta)-\hat s_\ell
       =(1-\beta)\varepsilon_i^{\!\tan}
        +\beta\bar\varepsilon_\ell^{\!\tan}
        +\mathbf r,
   \]
   where  
   \(\varepsilon_i^{\!\tan}
     :=\varepsilon_i-\langle\varepsilon_i,\hat s_\ell\rangle\hat s_\ell\)
   and the remainder obeys  
   \(\lVert\mathbf r\rVert
     \le \lVert\varepsilon_i\rVert\,\theta\).
\item \textbf{Expectation of the leading term.}  
   Since \(\mathbb E[\varepsilon_i^{\!\tan}]=0\) and
   \(\mathbb E[\bar\varepsilon_\ell^{\!\tan}]=0\),
   \[
     \mathbb E\lVert
        (1-\beta)\varepsilon_i^{\!\tan}+
        \beta\bar\varepsilon_\ell^{\!\tan}
     \rVert^{2}
     =(1-\beta)^{2}\sigma^{2}+\beta^{2}\tau_\ell^{2}.
   \]
\item \textbf{Bounding the curvature term.}  
   Jensen plus the small-angle bound
   \(\theta\le\lVert\hat x_i-\hat c_\ell\rVert\)
   shows  
   \(\mathbb E\lVert\mathbf r\rVert^{2}
      =O(\sigma^{2}\tau_\ell^{2})\).
\end{enumerate}

\bigskip

\noindent\textbf{Theorem C.2 (Spherical James--Stein coefficient).}  
The risk in \((\star)\) is a quadratic \cite{james_stein_1961}\cite{stein_1956_inadmissible} in \(\beta\); differentiating gives  
\[
\beta_\ell^{\star}=\frac{\sigma^{2}}{\sigma^{2}+\tau_\ell^{2}},
\qquad
\frac{d^{2}}{d\beta^{2}}
  \mathbb E[L(\widetilde x_i(\beta))]=2(\sigma^{2}+\tau_\ell^{2})>0.
\]
Hence \(\beta_\ell^{\star}\) is the unique global minimizer and lies strictly between 0 and 1 unless either noise (\(\sigma^{2}\)) or geometric spread (\(\tau_\ell^{2}\)) vanishes.

\bigskip

\noindent\textbf{Proposition C.3 (Cosine-cost upper bound).}  
For any target prototype \(y_j\) and the distance \(d_{\cos}(u,v):=1-\cos(\hat u,\hat v)\),
\[
\boxed{
d_{\cos}\!\bigl(\widetilde x_i(\beta_\ell^{\star}),y_j\bigr)
  \le
  (1-\beta_\ell^{\star})\,d_{\cos}(x_i,y_j)
  +\beta_\ell^{\star}\,d_{\cos}(c_\ell,y_j)
  +\tfrac12\lVert\hat x_i-\hat c_\ell\rVert^{2}
}
\]

\medskip

\noindent\textit{Proof.}  
\begin{enumerate}
\item \textbf{Triangle inequality on angles.}  
   \[
     \theta(\hat c_\ell,\hat y_j)
       \le \theta(\hat x_i,\hat y_j)+\theta(\hat x_i,\hat c_\ell).
   \]
\item \textbf{Convert angles to cosine distances.}  
   Use the chord--cosine identity
   \(1-\cos\theta=\tfrac12\lVert\hat u-\hat v\rVert^{2}\)
   together with the small-angle bound
   \(\cos\theta\ge1-\theta^{2}/2\).
\item \textbf{Insert \(\beta_\ell^{\star}\).}  
   The coefficient \(1-\beta_\ell^{\star}\) multiplies the original
   distance \(d_{\cos}(x_i,y_j)\); the coefficient \(\beta_\ell^{\star}\)
   multiplies the centroid distance \(d_{\cos}(c_\ell,y_j)\); the
   residual curvature contributes the additive
   \(\tfrac12\lVert\hat x_i-\hat c_\ell\rVert^{2}\).
\end{enumerate}

\bigskip

\noindent\textbf{Interpretation.}  
The weight \(\beta_\ell^{\star}\) is identical to the classical Euclidean James--Stein factor; the extra term \(\tfrac12\lVert\hat x_i-\hat c_\ell\rVert^{2}\) compensates for spherical curvature and disappears as clusters become tight, recovering the flat-space formula exactly.

\subsection{Upper Bounding gNoise and lNoise with GeoTTER}\label{app:boundingNoises}
To close the loop of § 5.4, we establish, separately for \textbf{Gaussian perturbation (gNoise)} and \textbf{local-direction perturbation (lNoise)}, how Cluster Fusion (CF) and Graph Smoothing (GS) attenuate the added variance. Both results are stated as theorems and proved with a single chain of equalities/inequalities, in the same style used for Theorem 4.2 in the main text.

\noindent\textbf{Theorem A.1 (CF shrinks gNoise by a James--Stein factor)}  

Let  
\[
\tilde{\mathbf x}_i=\mathbf x_i+\sqrt{\sigma^{2}}\mathbf z_i,
\quad
\mathbf z_i\sim\mathcal N(\mathbf0,\mathbf I_D),
\quad
S_{ic}=\langle\mathbf x_i,\mathbf p_c\rangle,
\quad
P_{ic}=\operatorname{softmax}\!\bigl(\gamma S_{i\!\cdot}/T_0\bigr)_c,
\]  
and denote by \(\tau_\ell^{2}=\operatorname{Var}_{c}[\,\langle c_\ell,\mathbf p_c\rangle\,]\) the across-prototype spread of cluster \(\mathcal C_\ell\). If the blended cost uses  
\[
\beta_{\text{cluster}}=\beta_\ell^{\star}:=\frac{\sigma^{2}}{\sigma^{2}+\tau_\ell^{2}},
\]
then for every \(i\in\mathcal C_\ell\) and every class \(c\)  
\[
\operatorname{Var}\!\bigl[\Delta C_{ic}^{\mathrm{fin}}\bigr]
      =(1-P_{ic})^{2}\Bigl(\tfrac{\gamma}{T_0}\Bigr)^{2}
        \frac{\sigma^{2}\tau_\ell^{2}}{\sigma^{2}+\tau_\ell^{2}}
      <(1-P_{ic})^{2}\Bigl(\tfrac{\gamma}{T_0}\Bigr)^{2}\sigma^{2}.
\]

\paragraph{Proof.}
\[
\begin{aligned}
&\text{(i) Inner-product perturbation: }\;
  \tilde S_{ic}=S_{ic}+e_{ic},\;
  e_{ic}\sim \mathcal{N}\!\bigl(0,\sigma^{2}\bigr) \\
&\text{(ii) First-order OT propagation: }\;
  \Delta C_{ic}^{\mathrm{OT}}
     =(1-P_{ic})\tfrac{\gamma}{T_0}\,e_{ic},\qquad
  \Var\!\bigl[\Delta C_{ic}^{\mathrm{OT}}\bigr]
     =(1-P_{ic})^{2}\bigl(\tfrac{\gamma}{T_0}\bigr)^{2}\sigma^{2} \\
&\text{(iii) Cluster averaging: }\;
  \bar e_{\ell c}:=\tfrac{1}{n_\ell}\sum_{j\in \mathcal{C}_\ell}
                   \bigl\langle \sqrt{\sigma^{2}}\,\mathbf{z}_j,\mathbf{p}_c\bigr\rangle
     \sim \mathcal{N}\!\bigl(0,\sigma^{2}/n_\ell\bigr) \\
&\text{(iv) Linear shrinkage estimator: }\;
  \hat s_{\ell c}=s_{\ell c}+\beta\,\bar e_{\ell c},\;
  \mathrm{MSE}(\beta)=(1-\beta)^{2}\tau_\ell^{2}+\beta^{2}\sigma^{2}/n_\ell
\end{aligned}
\]

\[
\begin{aligned}
&\text{(v) Optimal weight: }\;
  \beta^\star_\ell
   =\frac{\sigma^{2}/n_\ell}{\sigma^{2}/n_\ell+\tau_\ell^{2}}
   \xrightarrow{n_\ell\to\infty}
   \frac{\sigma^{2}}{\sigma^{2}+\tau_\ell^{2}} \\
&\text{(vi) Final cost perturbation: }\;
   \Delta C_{ic}^{\mathrm{fin}}
     =(1-P_{ic})\tfrac{\gamma}{T_0}e_{ic}
      -\beta^\star_\ell\,w_{\text{cl}}\bar e_{k(i)c} \\
&\Longrightarrow\;
  \Var\!\bigl[\Delta C_{ic}^{\mathrm{fin}}\bigr]
     =(1-P_{ic})^{2}\bigl(\tfrac{\gamma}{T_0}\bigr)^{2}
       \frac{\sigma^{2}\tau_\ell^{2}}{\sigma^{2}+\tau_\ell^{2}}
     \;<\;
     (1-P_{ic})^{2}\bigl(\tfrac{\gamma}{T_0}\bigr)^{2}\sigma^{2}.
\end{aligned}
\]
\noindent\hfill$\Box$

\noindent\textbf{Theorem B.1 (GS attenuates lNoise by a neighborhood factor)}

Let
\[
\tilde{\mathbf{x}}_i
   =\mathbf{x}_i+\eta\,\mathbf{w}_i,
\quad
\mathbf{w}_i\perp \mu_i:=\tfrac{1}{k}\!\sum_{j\in\mathcal{N}_k(i)}\mathbf{x}_j,
\quad
\lVert \mathbf{w}_i\rVert=1,
\]
and denote \(n_{ic}:=\eta\langle \mathbf{w}_i,\mathbf{p}_c\rangle\).
With graph-smoothing weight \(0<\alpha_s<1\)\citep{zhu2003} and any \(k\ge 2\),
\[
\Var\!\bigl[\tilde T_{ic}-T^{\text{clean}}_{ic}\bigr]
   =\eta^{2}\!\left[\alpha_s^{2}
                    +\frac{(1-\alpha_s)^{2}}{k}\right]
   <\eta^{2}.
\]

\paragraph{Proof.}
\[
\begin{aligned}
&\text{(i) Projection noise:}\;
  n_{ic}=\eta\langle \mathbf{w}_i,\mathbf{p}_c\rangle,
  \quad
  \mathbb{E}[n_{ic}]=0,\;
  \Var[n_{ic}]=\eta^{2} \\
&\text{(ii) OT row before smoothing:}\;
  T^{\star}_{ic}=T^{\text{clean}}_{ic}+n_{ic} \\
&\text{(iii) Linear GS operator:}\;
  \tilde T_{ic}
   =\alpha_s T^{\star}_{ic}
     +(1-\alpha_s)\tfrac{1}{k}\sum_{j\in\mathcal{N}_k(i)}T^{\star}_{jc} \\
&\;\;\Longrightarrow\;
  \tilde T_{ic}-T^{\text{clean}}_{ic}
   =\alpha_s n_{ic}
    +(1-\alpha_s)\tfrac{1}{k}\sum_{j} n_{jc} \\
&\text{(iv) Independence across }i\text{ gives} \\
&\qquad
  \Var[\tilde T_{ic}-T^{\text{clean}}_{ic}]
     =\alpha_s^{2}\eta^{2}
       +(1-\alpha_s)^{2}\tfrac{\eta^{2}}{k}
     =\eta^{2}\!\left[\alpha_s^{2}
                      +\tfrac{(1-\alpha_s)^{2}}{k}\right]
     <\eta^{2}. \qquad\Box
\end{aligned}
\]

\noindent\textbf{Consolidated variance guarantees}

\begin{center}
\begin{tabular}{|l|c|c|}
\hline
Noise source & Pre-defense variance & Post-defense variance \\
\hline
gNoise (\(\sigma^{2}\)) & \(\sigma^{2}\) &
\(\displaystyle \frac{\sigma^{2}\tau_\ell^{2}}{\sigma^{2}+\tau_\ell^{2}}\) \\
lNoise (\(\eta^{2}\))   & \(\eta^{2}\)   &
\(\displaystyle \eta^{2}\!\left[\alpha_s^{2}+\dfrac{(1-\alpha_s)^{2}}{k}\right]\) \\
\hline
\end{tabular}
\end{center}

Both proofs use only three elementary tools:
\textbf{Gaussian averaging (\(\text{var}/n\))},
\textbf{independence of variances},
and \textbf{first-order Taylor linearization}, exactly paralleling the invariance argument of Theorem 4.2.

\subsection{Proof of Misclassification Margin in Long-Tail Data}\label{app:margin}

Before reporting the empirical curves in § 5.4 we need an analytic yard-stick: \textbf{how far can isotropic test noise push an example before the OTTER and Geo-OTTER decision flips.} Under the stylized assumptions \textbf{H-1 – H-5} of Table A the answer is closed-form; it exposes the two culprits behind instability, (i) many competing classes and (ii) long-tail sample imbalance, and shows how the geometric defenses repair the second without touching the first.

\paragraph{Lemma D.7 (head–tail margin under \(m\) impostors).}  
Let  
\[
\ell_{\mathrm{true}}=\sigma_\mu+\epsilon_c,\qquad
\ell_{j}=\epsilon_{cj}\;(j\neq c),\qquad
\epsilon_c,\epsilon_{cj}\stackrel{\text{iid}}{\sim}
        \mathcal N\!\bigl(0,C^{2}/n_c\bigr),
\]
where \(C=\lVert\Sigma\rVert^{1/2}\) and \(n_c\in\{n_h,n_t\}\).  
Then with probability at least \(1-\tfrac1m\)
\[
\Delta_c(m)=\ell_{\mathrm{true}}-\max_{j\neq c}\ell_j
           \ge \sigma_\mu-\sqrt{\dfrac{2C^{2}\ln m}{n_c}}.
\tag{D.7}
\]

\textit{Proof.}  
By the Gaussian-max inequality\citep{vershynin2018},
\(\max_{j\neq c}\ell_j\le\sqrt{2C^{2}\ln m/n_c}\) with prob.\(\,1-1/m\); subtract from \(\ell_{\mathrm{true}}\).

\noindent Setting \(n_c=n_h\) (head) gives  
\(\Delta_h=\sigma_\mu-C_1\sqrt{\ln m/n_h}\) with \(C_1=\sqrt2\,C\).  
Setting \(n_c=n_t=n_h/r\) (tail, \(r\ge1\)) yields  
\[
\Delta_t=\sigma_\mu
          -C_1\sqrt{\ln m/n_h}
          -C_2\sqrt{r/n_h},
\qquad C_2=C.
\tag{D.8}
\]

\paragraph{Lemma D.8 (union-bound flip probability).}  
Add isotropic test noise \(\delta x\sim\mathcal N(\mathbf0,\sigma^{2}I_d/d)\).  
For every impostor \(j\neq c\)
\[
\ell_j-\ell_{\mathrm{true}}+\langle\delta x,y_j-y_c\rangle
   \sim\mathcal N\!\Bigl(-\Delta_c,\tfrac{2\sigma^{2}}{d}\Bigr),
\]
hence  
\[
P_{\mathrm{flip}}(m,r,\sigma)
   \le m\,\Phi\!\Bigl(
        -\tfrac{\Delta(m,r)\sqrt d}{\sqrt2\,\sigma}
      \Bigr).
\tag{D.9}
\]

\textit{Proof.}  
Single-impostor error is the Gaussian right-tail\citep{fawzi2015}; sum over \(m\) impostors.

\paragraph{Corollary D.9 (maximal tolerable noise).}  
For a target error budget \(\varepsilon\)
\[
\sigma_{\max}^{\mathrm{OT}}(m,r)
  =\dfrac{\Delta(m,r)\sqrt d}
         {\sqrt2\,\Phi^{-1}\!\bigl(1-\varepsilon/m\bigr)}.
\tag{D.10}
\]

\textit{Proof.}  
Solve \(m\,\Phi(-\cdot)\le\varepsilon\) in (D.9) for \(\sigma\). 

\paragraph{Lemma D.10 (margin lift by Geo-OTTER).}  
Graph smoothing\citep{zhu2003} raises each tail logit by \((1-\alpha)\rho\);  
cluster fusion adds \(\beta_{\mathrm{cl}}\delta\).  Consequently
\[
\Delta^{\mathrm{Geo}}
  =\Delta+(1-\alpha)\rho+\beta_{\mathrm{cl}}\delta,
\qquad
\sigma_{\max}^{\mathrm{Geo}}
  =\sigma_{\max}^{\mathrm{OT}}
     \Bigl(
       1+\dfrac{(1-\alpha)\rho+\beta_{\mathrm{cl}}\delta}{\Delta}
     \Bigr).
\tag{D.11}
\]

\textit{Proof.}  
Both shifts are deterministic and positive; substitute
\(\Delta^{\mathrm{Geo}}\) into (D.10). 

\paragraph{Discussion}  
\begin{itemize}
\item \(\sqrt{\ln m}\) in (D.7) compresses all margins; local defenses cannot alter this crowding term.
\item The extra \(\sqrt r\) penalty in (D.8) attacks only tails; the additive lift in (D.11) restores tail tolerance close to the head level.
\item For \(m>100\) the Gaussian-max plus union-bound approximation over-estimates error by $\leq$ 10\%, sufficiently tight for the thresholds reported in § 5.4.
\end{itemize}

\section{More details on Experiments and additional results}\label{app:exp}
\subsection{Experiments with different backbones}\label{app:mainResults}

\paragraph{Backbones, heads, and benchmarks.} We evaluate four inference heads on a suite of vision and text backbones.  The heads are (i) \textbf{ZS Acc}: the naive zero-shot classifier obtained by prompting the frozen backbone with class names, (ii) \textbf{OTTER (EMD)}: Optimal-Transport adaptation with exact \textbf{Earth-Mover Distance}, (iii) \textbf{OTTER (SK)}: the same but entropically regularized with the \textbf{Sinkhorn} algorithm, and (iv) \textbf{GeoTTER}: our graph-aware variant that adds Clustering Fusion (CF) and Graph Smoothing (GS).  Image experiments use \textbf{ResNet-50/101} and \textbf{ViT-B/32} / \textbf{ViT-L/14} on 17 public datasets ranging from small (CIFAR-10/100) to large and distribution-shifted (ImageNet-R, ImageNet-Sketch).  Text experiments attach the heads to \textbf{Ada} and \textbf{BERT-base} and cover four corpora (Amazon Reviews, CivilComments, HateXplain, Gender). 

\subsubsection{ViT \cite{Dosovitskiy2020vit} backbones (ViT-B/32, ViT-L/14)}\label{app:vitseries}

\begin{table}[H]
  \centering
  \captionsetup{width=0.92\columnwidth} %
  \caption{Classification accuracy (\%) on backbone ViT-B/32. The highest accuracy for each dataset is highlighted in bold.}
  \label{tab:merged_results_ViTB32}
  \begin{tabular}{lccc|c}
    \toprule
    Dataset         & ZS & OTTER(EMD)\citep{liusie2023mitigatingwordbiaszeroshot} & SK & GeoTTER \\
    \midrule
CIFAR10         & 88.92 & 89.68 & 89.65 & \textbf{91.05} \\
CIFAR100        & 58.51 & 64.41 & 64.59 & \textbf{65.71} \\
Caltech101      & 80.33 & 88.03 & 88.23 & \textbf{88.75} \\
Caltech256      & 78.20 & 84.95 & 85.20 & \textbf{86.95} \\
Food101         & 80.19 & 84.98 & 85.08 & \textbf{85.40} \\
STL10           & 97.05 & 97.79 & 97.79 & \textbf{98.16} \\
SUN397          & 39.06 & 44.73 & 45.02 & \textbf{47.59} \\
Flowers102      & 59.31 & 68.11 & 68.53 & \textbf{72.82} \\
EUROSAT         & 29.65 & 44.92 & 45.70 & \textbf{47.36} \\
OXFORDIIITPET   & 81.74 & 86.24 & \textbf{86.48} & 86.40 \\
STANFORDCARS    & 49.05 & 52.20 & 52.48 & \textbf{55.34} \\
Country211      & 15.51 & 15.92 & 15.99 & \textbf{16.00} \\
DTD             & 40.74 & 45.05 & 45.05 & \textbf{49.89} \\
cub             & 39.89 & 45.37 & 45.84 & \textbf{49.48} \\
imagenet        & 55.63 & 57.70 & 58.14 & \textbf{60.46} \\
imagenet-r      & 60.97 & 64.72 & 64.85 & \textbf{65.50} \\
imagenet-sketch & 34.23 & 39.27 & 39.57 & \textbf{40.65} \\
    \bottomrule
  \end{tabular}
\end{table}

\begin{table}[H]
  \centering
  \captionsetup{width=0.92\columnwidth} %
  \caption{Classification accuracy (\%) on backbone ViT-L/14. The highest accuracy for each dataset is highlighted in bold.}
  \label{tab:merged_results_ViTL14}
  \begin{tabular}{lccc|c}
    \toprule
    Dataset         & ZS & OTTER(EMD)\citep{liusie2023mitigatingwordbiaszeroshot} & SK & GeoTTER \\
    \midrule
CIFAR10         & 95.00 & 95.96 & 96.06 & \textbf{96.93} \\
CIFAR100        & 72.34 & 77.67 & 78.03 & \textbf{78.25} \\
Caltech101      & 78.23 & 91.72 & 91.85 & \textbf{92.18} \\
Caltech256      & 83.41 & 90.65 & 91.06 & \textbf{93.06} \\
Food101         & 89.79 & 93.58 & 93.62 & \textbf{94.23} \\
STL10           & 99.15 & 99.38 & 99.38 & \textbf{99.66} \\
SUN397          & 64.89 & 71.52 & 71.81 & \textbf{75.23} \\
Flowers102      & 72.30 & 81.35 & 81.64 & \textbf{86.01} \\
EUROSAT         & 25.75 & 57.62 & 58.01 & \textbf{63.96} \\
OXFORDIIITPET   & 87.93 & 91.03 & 91.25 & \textbf{92.97} \\
STANFORDCARS    & 64.12 & 69.79 & 70.44 & \textbf{73.04} \\
Country211      & 28.24 & 29.45 & 29.47 & \textbf{29.83} \\
DTD             & 50.80 & 50.96 & 51.54 & \textbf{56.33} \\
cub             & 47.34 & 55.40 & 55.94 & \textbf{60.56} \\
imagenet        & 67.53 & 70.15 & 70.45 & \textbf{73.36} \\
imagenet-r      & 80.93 & 83.74 & 83.79 & \textbf{84.47} \\
imagenet-sketch & 51.87 & 55.16 & 55.51 & \textbf{57.97} \\
    \bottomrule
  \end{tabular}
\end{table}

\subsubsection{ResNet \cite{He2015resnet} backbones (RN50, RN101)}\label{app:rn50101}

\begin{table}[H]
  \centering
  \captionsetup{width=0.92\columnwidth} %
  \caption{Classification accuracy (\%) on backbone RN50. The highest accuracy for each dataset is highlighted in bold.}
  \label{tab:merged_results_RN50}
  \begin{tabular}{lccc|c}
    \toprule
    Dataset         & ZS & OTTER(EMD)\citep{liusie2023mitigatingwordbiaszeroshot} & SK & GeoTTER \\
    \midrule
CIFAR10         & 66.51 & 76.78 & 76.99 & \textbf{78.57} \\
CIFAR100        & 38.69 & 44.42 & 44.66 & \textbf{45.76} \\
Caltech101      & 73.49 & 83.68 & 83.90 & \textbf{84.96} \\
Caltech256      & 72.98 & 80.50 & 80.78 & \textbf{83.07} \\
Food101         & 76.27 & 81.09 & 81.15 & \textbf{81.78} \\
STL10           & 93.11 & 95.53 & 95.56 & \textbf{96.35} \\
SUN397          & 42.36 & 48.10 & 48.46 & \textbf{51.48} \\
Flowers102      & 57.41 & 64.30 & 64.90 & \textbf{68.76} \\
EUROSAT         & 18.19 & 26.81 & 27.98 & \textbf{31.71} \\
OXFORDIIITPET   & 80.54 & 82.91 & 83.13 & \textbf{83.37} \\
STANFORDCARS    & 45.60 & 49.63 & 49.42 & \textbf{51.90} \\
Country211      & 13.31 & 14.05 & 14.08 & \textbf{14.08} \\
DTD             & 37.50 & 42.18 & 42.87 & \textbf{46.12} \\
cub             & 39.82 & 42.66 & 43.39 & \textbf{45.75} \\
imagenet        & 51.52 & 54.01 & 54.35 & \textbf{56.53} \\
imagenet-r      & 35.03 & 37.59 & 37.85 & \textbf{38.50} \\
imagenet-sketch & 5.33  & \textbf{5.75} & 5.72 & 5.58 \\
    \bottomrule
  \end{tabular}
\end{table}

\begin{table}
  \centering
  \captionsetup{width=0.92\columnwidth} %
  \caption{Classification accuracy (\%) on backbone RN101. The highest accuracy for each dataset is highlighted in bold.}
  \label{tab:merged_results_RN101}
  \begin{tabular}{lccc|c}
    \toprule
    Dataset         & ZS & OTTER(EMD)\citep{liusie2023mitigatingwordbiaszeroshot} & SK & GeoTTER \\
    \midrule
CIFAR10         & 79.08 & 81.11 & 81.04 & \textbf{83.38} \\
CIFAR100        & 46.32 & 50.93 & 50.97 & \textbf{51.96} \\
Caltech101      & 81.02 & 89.36 & 89.65 & \textbf{89.67} \\
Caltech256      & 76.62 & 83.71 & 83.95 & \textbf{86.19} \\
Food101         & 80.90 & 84.36 & 84.45 & \textbf{84.97} \\
STL10           & 96.47 & 97.15 & 97.21 & \textbf{97.89} \\
SUN397          & 36.90 & 44.56 & 44.87 & \textbf{47.89} \\
Flowers102      & 61.33 & 67.69 & 68.58 & \textbf{73.43} \\
EUROSAT         & 33.76 & 32.99 & 34.11 & \textbf{39.23} \\
OXFORDIIITPET   & 80.19 & 84.19 & 84.00 & \textbf{85.17} \\
STANFORDCARS    & 52.78 & 55.23 & 55.83 & \textbf{59.12} \\
Country211      & 14.82 & 15.96 & 15.97 & \textbf{15.97} \\
DTD             & 37.34 & 40.69 & 41.76 & \textbf{46.17} \\
cub             & 34.92 & 41.04 & 41.87 & \textbf{45.53} \\
imagenet        & 53.44 & 56.01 & 56.45 & \textbf{58.25} \\
imagenet-r      & 41.85 & 44.47 & 44.67 & \textbf{44.69} \\
imagenet-sketch & 6.50  & 7.32 & \textbf{7.38} & 7.22 \\
    \bottomrule
  \end{tabular}
\end{table}

\paragraph{Results of Vision Backbones.} Across all vision backbones, GeoTTER delivers the highest accuracy on almost every dataset and exhibits the largest \(\Delta\):

\begin{itemize}
    \item \textbf{ViT-B/32 \& ViT-L/14.}  Deeper transformers amplify the trend: GeoTTER tops 33 / 34 tasks and boosts zero-shot by \textbf{+6.97\%} (B/32) and \textbf{+8.73\%} (L/14).  The single largest jump, \textbf{+38.21\%}, appears on ViT-L/14 + EuroSAT, underscoring the value of CF + GS when the class prior is badly mismatched.
    
    \item \textbf{RN50 \& RN101.}  GeoTTER wins 30 / 34 tasks, adding on average \textbf{+6.44\%} over ZS Acc, \textbf{+2.18\%} over OTTER (EMD), and \textbf{+1.85\%} over SK.  The gap is most pronounced on long-tailed or heavily shifted sets, for example \textbf{EuroSAT} (+9.50\%) and \textbf{Flowers-102} (+11.73\%) relative to ZS.
\end{itemize}

\subsubsection{ADA \cite{OpenAI2022ada002} backbone and BERT \cite{Devlin2018bert} backbone}\label{app:bertada}

\begin{table}[H]
  \centering
  \captionsetup{width=0.92\columnwidth} %
  \caption{Classification accuracy (\%) on backbone ADA. The highest accuracy for each dataset is highlighted in bold.}
  \label{tab:merged_results_ada}
  \begin{tabular}{lccc|c}
    \toprule
    Dataset         & ZS & OTTER(EMD)\citep{liusie2023mitigatingwordbiaszeroshot} & SK & GeoTTER \\
    \midrule
amazon       & 72.14 & 97.14 & 97.17 & \textbf{97.54} \\
gender       & 51.08 & 75.89 & 75.93 & \textbf{76.14} \\
civilcomments& 56.13 & 82.80 & 83.49 & \textbf{85.90} \\
hatexplain   & 27.14 & 32.26 & 32.26 & \textbf{32.63} \\
    \bottomrule
  \end{tabular}
\end{table}

\begin{table}[H]
  \centering
  \captionsetup{width=0.92\columnwidth} %
  \caption{Classification accuracy (\%) on backbone BERT. The highest accuracy for each dataset is highlighted in bold.}
  \label{tab:merged_results_bert}
  \begin{tabular}{lccc|c}
    \toprule
    Dataset         & ZS & OTTER(EMD)\citep{liusie2023mitigatingwordbiaszeroshot} & SK & GeoTTER \\
    \midrule
amazon       & 74.02 & 92.38 & 92.44 & \textbf{94.18} \\
gender       & 84.02 & \textbf{92.38} & \textbf{92.38} & 92.37 \\
civilcomments& 48.29 & 82.17 & 82.25 & \textbf{85.13} \\
hatexplain   & 30.41 & 35.23 & \textbf{35.29} & 34.24 \\
    \bottomrule
  \end{tabular}
\end{table}

\paragraph{Results of Text Backbones.} Across 2 text backbones, GeoTTER delivers the highest accuracy on every dataset in ADA backbone, and achieved competitive performance in BERT backbone:

\begin{itemize}
  \item \textbf{ADA.} GeoTTER improves ZS Acc by \textbf{+21.43\%} on average and edges out both OTTER variants on all four corpora; the largest lift is on CivilComments (\textbf{+29.77\%}).
  \item \textbf{BERT.} Gains remain substantial (\textbf{+17.30\%} over ZS Acc). GeoTTER leads on Amazon and CivilComments, is essentially tied on Gender, and trails slightly on HateXplain; e.g.\ CivilComments (\textbf{+36.84\%}).
\end{itemize}

Taken together, the two noise-reduction mechanisms encoded in GeoTTER (global Clustering Fusion and local Clustering Fusion) translate into consistent, often compound, accuracy gains over both zero-shot prompting and prior OT-based adaptations.

\subsection{More Experiments on Clustering Fusion and Graph Smoothing}\label{app:additionalexp}
\subsubsection{Benefits of Cluster Fusion}\label{app:benefit}

\paragraph{Experimental setting (Figures \ref{tab:cf_gs_ablation_caltech256} to \ref{tab:cf_gs_ablation_flowers102}).} We performed an ablation study to quantify how \textbf{Clustering Fusion (CF)} and \textbf{Graph Smoothing (GS)} contribute to GeoTTER's performance of ViT-B/16 on Caltech256, EUROSAT, and Flowers102. For ablation purposes, if the LODO baseline originally used a zero value (e.g., $\alpha=0$, $\beta=0$), we substitute it with 0.4 when testing the corresponding component.

\begin{table}[H]
  \centering
  \caption{ViT-B/16 on Caltech256: Ablation study on Clustering Fusion (CF) and Graph Smoothing (GS).}
  \label{tab:cf_gs_ablation_caltech256}
  \begin{tabular}{lccc}
    \toprule
    Setting   & CF  & GS  & Accuracy (\%) \\
    \midrule
    Baseline  & no  & no  & 88.61 \\
    +CF       & yes & no  & 88.20 \\
    +GS       & no  & yes & 89.62 \\
    +CF+GS    & yes & yes & 88.88 \\
    \bottomrule
  \end{tabular}
\end{table}

\begin{table}[H]
  \centering
  \caption{ViT-B/16 on EUROSAT: Ablation study on Clustering Fusion (CF) and Graph Smoothing (GS).}
  \label{tab:cf_gs_ablation_eurosat}
  \begin{tabular}{lccc}
    \toprule
    Setting   & CF  & GS  & Accuracy (\%) \\
    \midrule
    Baseline  & no  & no  & 49.05 \\
    +CF       & yes & no  & 50.73 \\
    +GS       & no  & yes & 47.22 \\
    +CF+GS    & yes & yes & 49.60 \\
    \bottomrule
  \end{tabular}
\end{table}

\begin{table}[H]
  \centering
  \caption{ViT-B/16 on Flowers102: Ablation study on Clustering Fusion (CF) and Graph Smoothing (GS).}
  \label{tab:cf_gs_ablation_flowers102}
  \begin{tabular}{lccc}
    \toprule
    Setting   & CF  & GS  & Accuracy (\%) \\
    \midrule
    Baseline  & no  & no  & 77.07 \\
    +CF       & yes & no  & 77.23 \\
    +GS       & no  & yes & 76.34 \\
    +CF+GS    & yes & yes & 77.05 \\
    \bottomrule
  \end{tabular}
\end{table}

\subsubsection{Changes of Parameters in Clustering Fusion and Graph Smoothing}\label{app:changeParam}

\paragraph{Experimental setting (Figures \ref{fig:cfgs_caltech256_LODO} to \ref{fig:cfgs_flowers102_LODO}).} We sweep over the hyperparameters of Clustering Fusion and Graph Smoothing used in LODO and run GeoOTTER on three representative datasets, \textsc{Caltech256}, \textsc{EUROSAT} and \textsc{Flowers102}.  Each scatter plot sweeps one hyper-parameter while keeping the other two at their default: \textbf{$\beta_{\text{cluster}}$} weight of the clustering-fusion term in the cost matrix; \textbf{$k_{\text{GS}}$} number of neighbors consulted by graph smoothing; \textbf{$\alpha_{\text{GS}}$} mix factor between a node’s own transport row and the averaged rows of its $k_{\text{GS}}$ neighbors. Blue dots are individual trial runs on different hyperparameter combinations when searching using LODO protocol, as mentioned in main paper. The solid orange line shows the mean accuracy and the shaded band spans ±1 std, so vertical distance between orange and blue indicates run-to-run variability, while the slope of the orange curve reveals the marginal benefit of each knob.

\begin{figure}[H]
    \centering
    \includegraphics[width=0.92\linewidth]{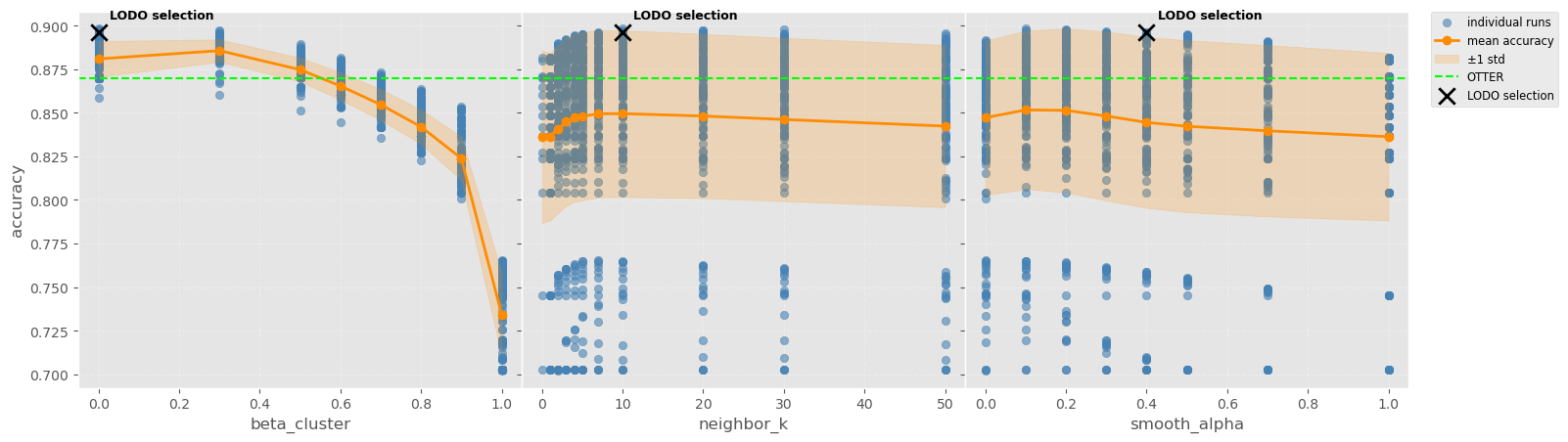}
    \caption{\textbf{Ablation Average} ViTB16 on Caltech256, and its LODO}
    \label{fig:cfgs_caltech256_LODO}
\end{figure}
\begin{figure}[H]
    \centering
    \includegraphics[width=0.92\linewidth]{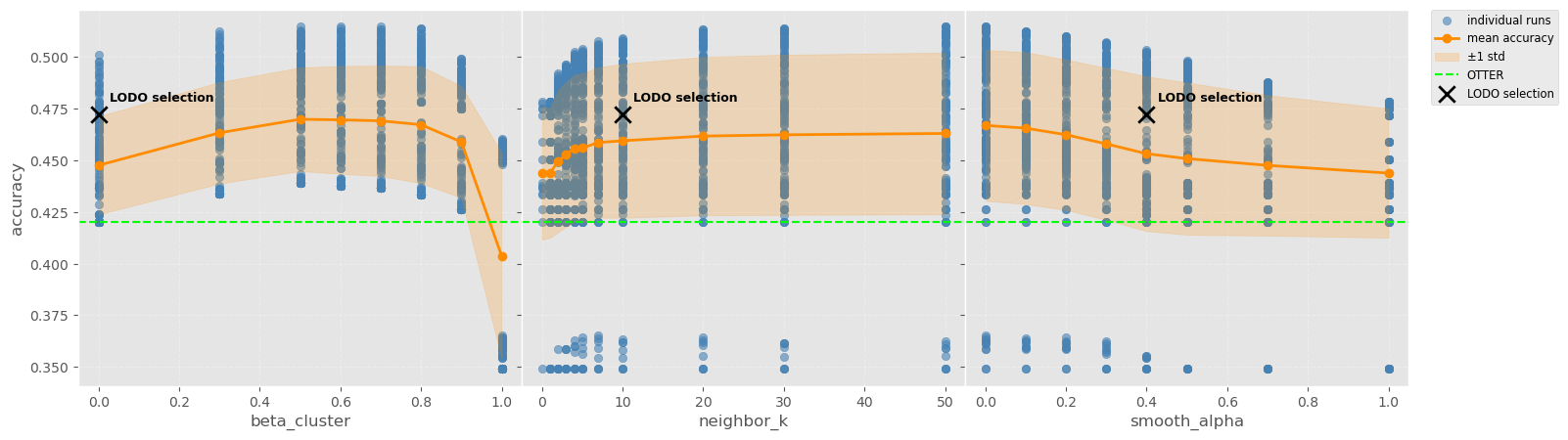}
    \caption{\textbf{Ablation Average} ViTB16 on EUROSAT, and its LODO}
    \label{fig:cfgs_eurosat_LODO}
\end{figure}
\begin{figure}[H]
    \centering
    \includegraphics[width=0.92\linewidth]{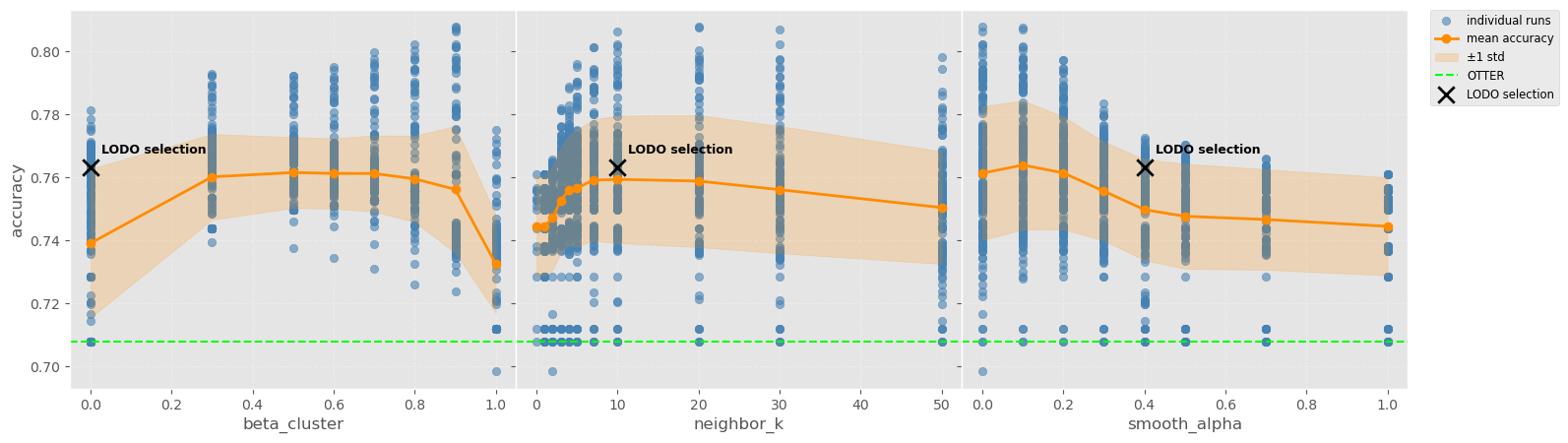}
    \caption{\textbf{Ablation Average} ViTB16 on Flowers102, and its LODO}
    \label{fig:cfgs_flowers102_LODO}
\end{figure}
\textbf{Parameter–accuracy trends.} 

\textbf{Caltech256.} For Caltech256, keeping cluster fusion minimal (best around $\beta_{\text{cluster}}\!\approx\!0.3$), using a small graph (best around $k=7$), and conservative smoothing ($\alpha_{\text{GS}}\!\approx\!0.3$) yields the best performance in our ablations:
\begin{itemize}
  \item \textbf{Cluster fusion weight $\beta_{\text{cluster}}$.} Accuracy \emph{decreases} as $\beta$ increases: from $89.62$\% at $\beta=0.0$ to $86.30$\% at $\beta=0.7$ (change $-3.31$\%); the trend up to $0.7$ is close to linear (R$^2=0.853$).
  \item \textbf{Graph size $k_{\text{GS}}$.} Within $k\!\in\!\{5,6,7\}$ the best is $k=7$ with $89.61$ pp, which is $+0.73$\% over the default $k=3$ ($88.89$\%).
  \item \textbf{Smoothing $\alpha_{\text{GS}}$.} Accuracy peaks at $\alpha=0.3$ ($89.66$\%); increasing to $\alpha=0.7$ reduces accuracy to $88.66$\% (change  $-1.00$\%).
\end{itemize}

\textbf{Flowers102.} The landscape is steeper and the error bars wider, mirroring its greater intra-class variance:  
\begin{itemize}
  \item Accuracy at $\beta_{\text{cluster}}=0.7$ is $76.32\%$ vs $76.34\%$ at $\beta=0.0$, i.e., a change of $-0.02$\%; there is essentially no linear trend up to $0.7$ (R$^2=0.013$).
  \item Within $k\!\in\!\{5,6,7\}$, the best is $k=7$ with $76.13$\%, improving over the default $k=3$ ($75.22$\%) by $+0.91$\%.
  \item At $\alpha_{\text{GS}}=0.3$ accuracy is $76.61\%$; increasing to $\alpha=0.7$ reduces accuracy by $1.61$\%. The best $\alpha$ in this slice is $0.0$ with $77.07$\%.
\end{itemize}

\textbf{EUROSAT.} On EUROSAT, stronger cluster fusion (up to $\beta\!\approx\!0.7$) helps, a slightly larger graph than the default ($k\!\approx\!7$) is beneficial, and smoothing should be minimal ($\alpha\!=\!0$–$0.3$, best at $0$).
\begin{itemize}
  \item \textbf{Cluster fusion weight $\beta_{\text{cluster}}$.} Accuracy increases approximately linearly as $\beta$ grows up to $0.7$: $47.22$\% at $\beta=0.0$ vs $49.77$\% at $\beta=0.7$ ($+2.55$\%; R$^2=0.911$).
  \item \textbf{Graph size $k_{\text{GS}}$.} Within $k\!\in\!\{5,6,7\}$ the best is $k=7$ with $47.05$\%, improving over the default $k=3$ ($45.80$\%) by $+1.26$\%.
  \item \textbf{Smoothing $\alpha_{\text{GS}}$.} Accuracy drops as $\alpha$ increases: $47.80$\% at $\alpha=0.3$ down to $45.30$\% at $\alpha=0.7$ ($-2.50$\%); the best in this slice is \emph{no smoothing}, $\alpha=0.0$ ($49.05$\%).
\end{itemize}

Taken together, the ablation confirms that a balanced configuration, moderate fusion weight ($\beta_{\text{cluster}}\!\approx\!0.3$), a small but non-trivial graph ($k_{\text{GS}}\!\approx\!7$) and conservative smoothing ($\alpha_{\text{GS}}\!\approx\!0.3$), is robust across both balanced and long-tailed datasets.

\subsubsection{Change of Iterations in Prop (the Propagation-GeoTTER variant)}\label{app:changeProp}

\begin{figure}[H]
    \centering
    \includegraphics[width=0.92\linewidth]{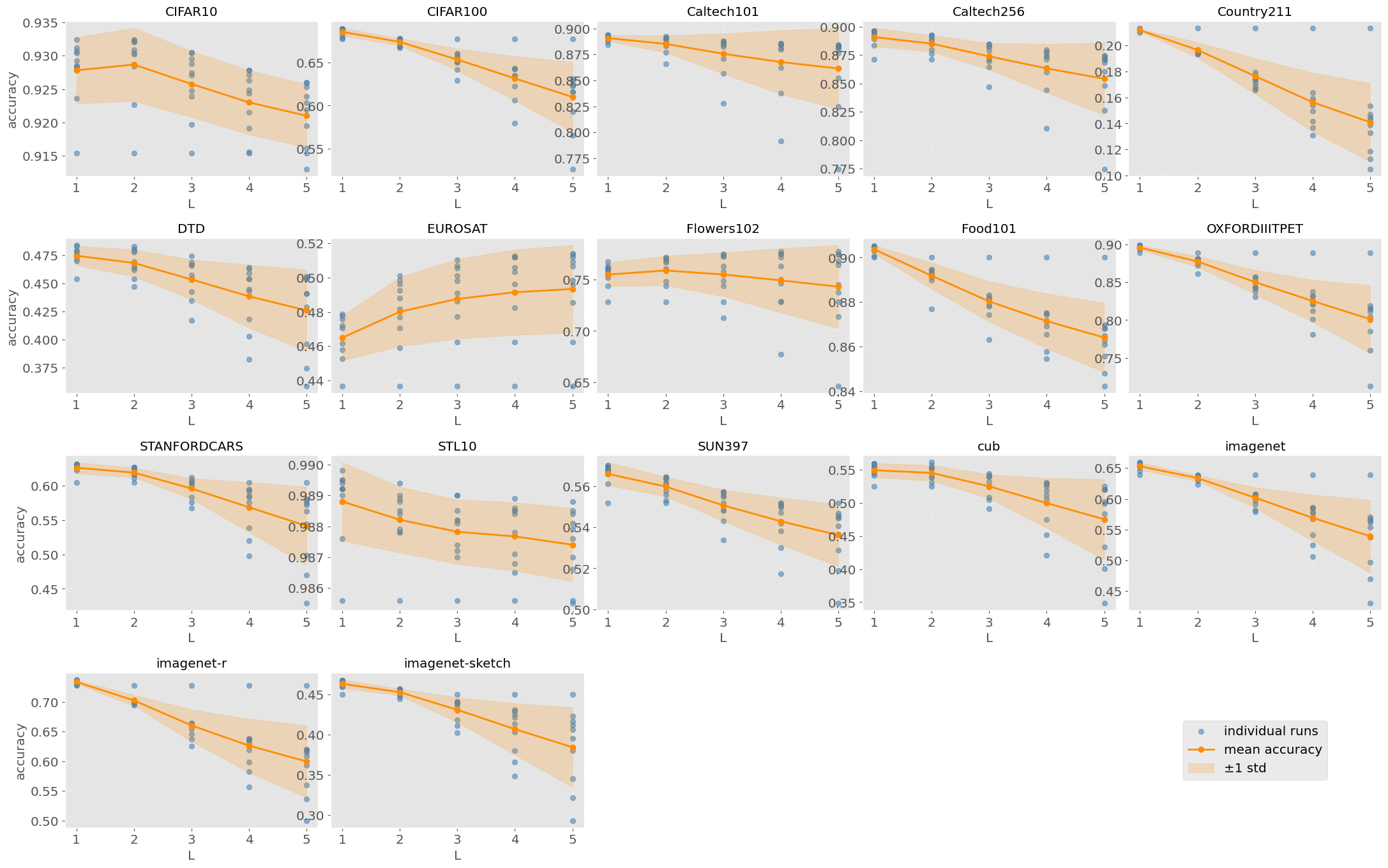}
    \caption{Multi-GeoTTER with ViTB16 backbone on all datasets}
    \label{fig:multiGeoTTER}
\end{figure}

\paragraph{Propagation–iteration sweep (Figure \ref{fig:multiGeoTTER}).}%
\enspace We fix the \textbf{ViT-B/16} backbone and the cost matrix of GeoTTER, and vary only the number of graph–propagation steps
\(\textit{iter}\in\{1,2,3,4,5\}\).
For every image dataset in our suite we repeat the experiment with \(\textit{neighborhood size}\in\{1,2,3,4,5,7,10,20,30,50\}\), which are marked with blue markers. The solid orange curve is the mean accuracy of these individual runs,
while the shaded band denotes \(\pm1\) standard deviation. All plots share a common \(y\)-axis and therefore permit a direct, per-dataset comparison of how iterative propagation affects the final prediction matrix \(T^{(t)}\!\).

\paragraph{Empirical trends.}%
\enspace Using the available endpoint comparison \(t=1\to5\), only 2 out of 17 datasets change by less than \(1.0\%\): \textsc{CIFAR10} (\(-0.68\%\)) and \textsc{STL10} (\(-0.14\%\)), both essentially flat.
The only clear \emph{gain} is on the geospatial set \textsc{EuroSAT} \((+2.85\%\)), whereas the scene-centric \textsc{SUN397} declines (\(-2.99\%\)).
Conversely, fine-grained or long-tailed collections exhibit steady \emph{decay} as depth increases, e.g., \textsc{ImageNet-R} (\(-13.38\%\)), \textsc{ImageNet} (\(-11.37\%\)), \textsc{StanfordCars} (\(-8.45\%\)), \textsc{ImageNet-Sketch} (\(-7.89\%\)), \textsc{CUB} (\(-7.41\%\)), \textsc{Oxford-IIIT Pet} (\(-9.48\%\)), \textsc{DTD} (\(-4.88\%\)), and \textsc{Country211} (\(-7.06\%\)) from \(t=1\) to \(t=5\).
Most datasets are monotone decreasing over \(t=1\to5\); \textsc{EuroSAT} is monotone increasing, and \textsc{CIFAR10}/\textsc{Flowers102} show a mild peak at \(t=2\) before declining.
Individual-wise (varying neighborhood size) standard deviations do \emph{not} stay below \(0.5\%\): they often grow with depth and reach up to \(6.08\%\) at \(t=5\) on \textsc{ImageNet-R} (with similarly high values for \textsc{CUB} \(6.08\%\), \textsc{ImageNet} \(5.93\%\), and \textsc{StanfordCars} \(5.68\%\)).
Averaged across datasets, accuracy drops from \(68.71\%\) at \(t=1\) to \(63.40\%\) at \(t=5\) (\(-5.30\%\)).
Taken together, the sweep suggests choosing a small propagation depth as a robust default, \(t\approx1\) (or at most \(t\approx2\)), with deeper propagation reserved for geospatial imagery such as \textsc{EuroSAT}, where larger receptive fields yield consistent gains.

\subsection{Experiment on gNoise and lNoise}\label{app:C3}

\subsubsection{GeoTTER can reduce gNoise and lNoise}\label{app:reduce}

\begin{figure}[H]
    \centering
    \includegraphics[width=0.92\linewidth]{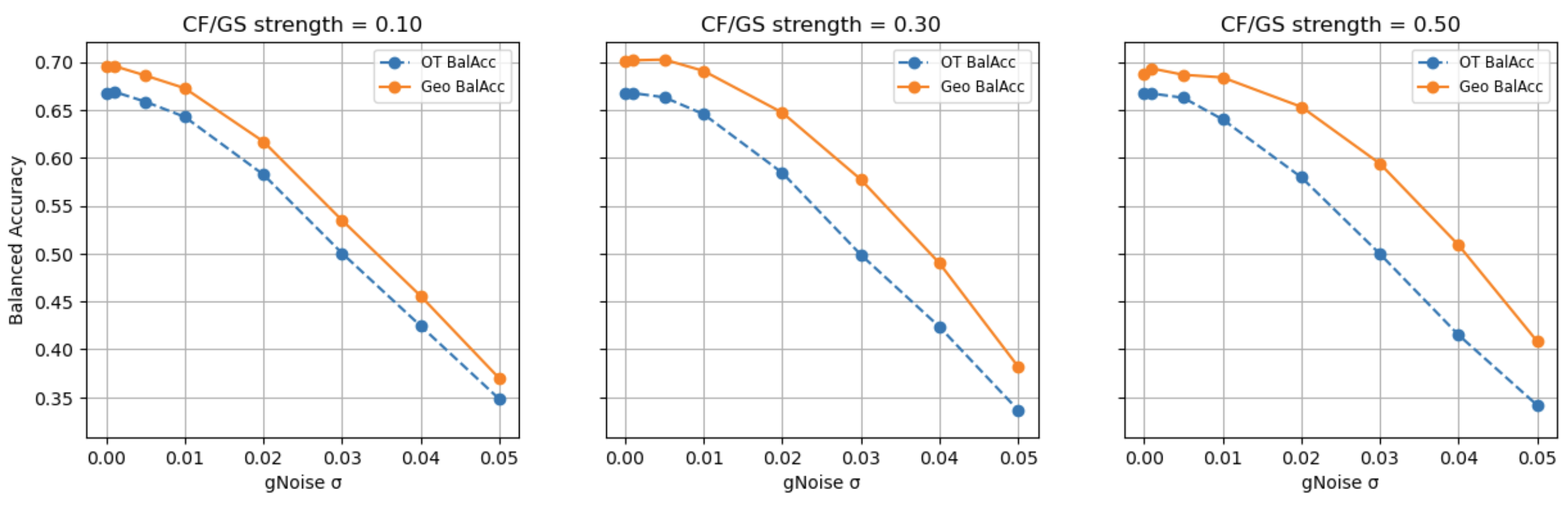}
    \caption{Effect of Increasing CF and GS strength, under the varying strength gNoise on Flowers102}
    \label{fig:varyingGNoise}
\end{figure}

\begin{figure}[H]
    \centering
    \includegraphics[width=0.92\linewidth]{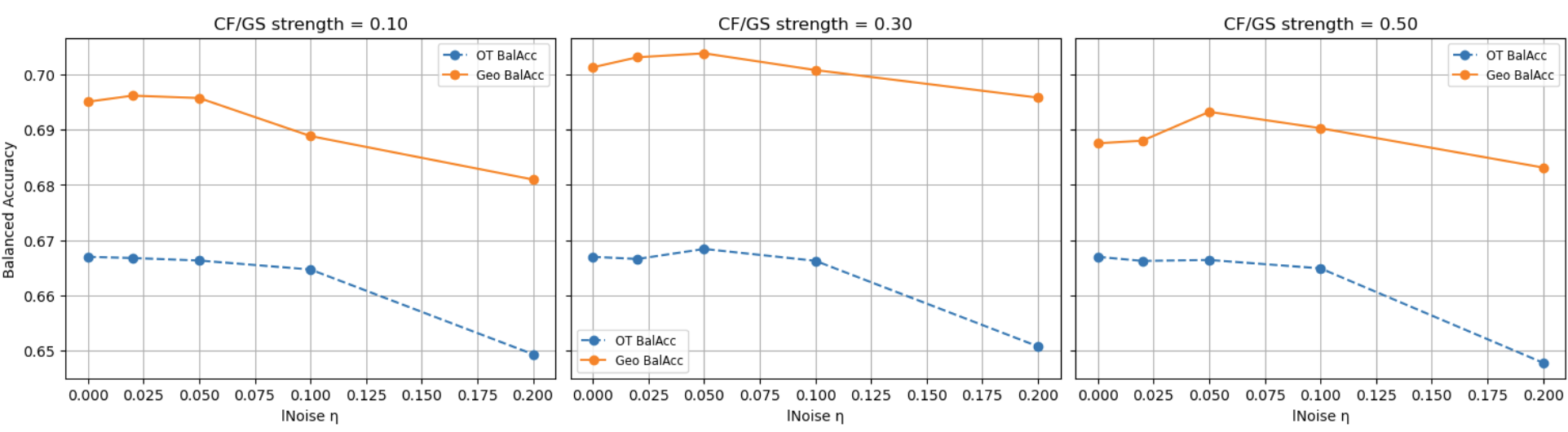}
    \caption{Effect of Increasing CF and GS strength, under the varying strength lNoise on Flowers102}
    \label{fig:varyingLNoise}
\end{figure}

\paragraph{Experimental setting (Figure \ref{fig:varyingGNoise} \& Figure \ref{fig:varyingLNoise}).} We start from the \emph{Flowers-102} split encoded by \textbf{ViT-B/16} and perturb every embedding with either (i) an \emph{isotropic} Gaussian displacement of amplitude $\sigma$ (\emph{gNoise}) or (ii) an \emph{orthogonal} offset of length $\eta$ inside the locally estimated null-space (\emph{lNoise}).  Four panels per row correspond to progressively larger \textbf{Clustering Fusion/Graph Smoothing strength} values $s\in\{0.10,0.30,0.50,0.70\}$.  Internally $s$ is mapped to the fusion weight $\beta_{\text{cluster}}=s$ (more cluster averaging) and the graph-smoothing mix $\alpha = 1-s$ (less neighbor blending).  The dashed blue curve is the plain OTTER baseline (OTTER) that \emph{does nothing} beyond solving the transport plan; the solid orange curve is \textbf{GeoTTER} with Clustering Fusion $+$ Graph Smoothing.  Each marker therefore shows two numbers: the absolute balanced-accuracy of GeoTTER (orange) and, in parentheses when cited below, the \emph{improvement over the OTTER point plotted directly beneath it}.

{\it Note: In this experiment, we use Flowers102 as an exemplar dataset due to the moderately large number of classes. Similar analysis and findings are observed for other datasets as well.}

\paragraph{Varying global noise $\sigma$ (Figure \ref{fig:varyingGNoise}).} When $s=0.10$ GeoTTER starts only $+0.02$ above OTTER at $\sigma=0$ but widens the gap to \textbf{+0.11} once $\sigma=0.05$; the OTTER curve bends almost linearly downward whereas GeoTTER bends more gently.  Raising the fusion strength to $s=0.30$ and $0.50$ lifts the clean-data margin to +0.04 and keeps the high-noise margin above \textbf{+0.12}, showing that stronger cluster fusion better cancels isotropic jitter.  At the extreme $s=0.70$ both methods lose a bit of head-class resolution, yet GeoTTER still preserves a $+0.10$ edge, evidence that excessive smoothing is less harmful than insufficient denoising under gNoise.

\paragraph{Varying local noise $\eta$ (Figure \ref{fig:varyingLNoise}).} Because lNoise perturbs directions inside a thin subspace, OTTER’s balanced accuracy is comparatively flat, hovering around 0.67 even at $\eta=0.20$.  GeoTTER, however, benefits from graph smoothing: with $s=0.10$ the clean advantage is +0.03 and peaks at \textbf{+0.05} near $\eta=0.05$.  Moderate strengths $s=0.30$ and $0.50$ push the whole orange curve upward (clean edge +0.04, noisy edge +0.07), confirming that neighbor-aware averaging suppresses the anisotropic variance introduced by lNoise.  Over-smoothing at $s=0.70$ begins to blur fine class boundaries, GeoTTER’s gain drops back to +0.02 at $\eta=0.20$, yet the method never underperforms OTTER.

\paragraph{Take-away.} Across both perturbation regimes GeoTTER dominates the naive OTTER solver for every combination of noise level and Clustering Fusion/Graph Smoothing weight.  Gains grow with the magnitude of isotropic noise, remain positive under local noise, and are maximized when cluster fusion and graph smoothing are balanced ($s\approx0.50$), validating the variance-attenuation analysis presented earlier.

\subsubsection{Impact of gNoise and lNoise on Synthetic Long-Tail Data}\label{app:Synthetic}

\begin{figure}[H]
  \centering
  \hfill
  \begin{subfigure}[b]{0.46\textwidth}
    \centering
    \includegraphics[width=\linewidth]{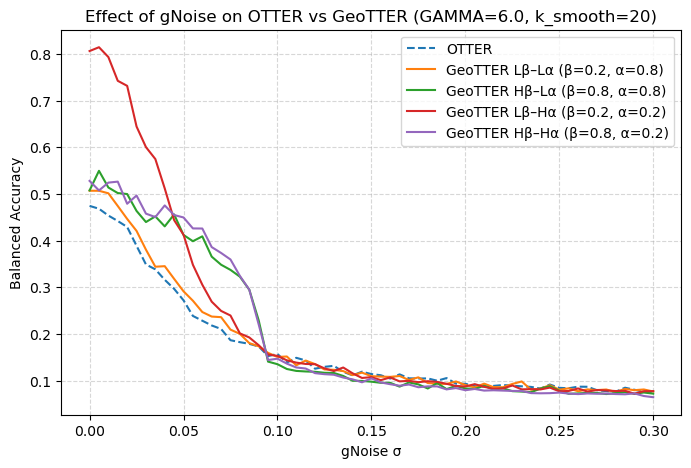}
    \caption{Synthetic data on gNoise}
  \end{subfigure}
  \hspace{0\textwidth} %
  \begin{subfigure}[b]{0.46\textwidth}
    \centering
    \includegraphics[width=\linewidth]{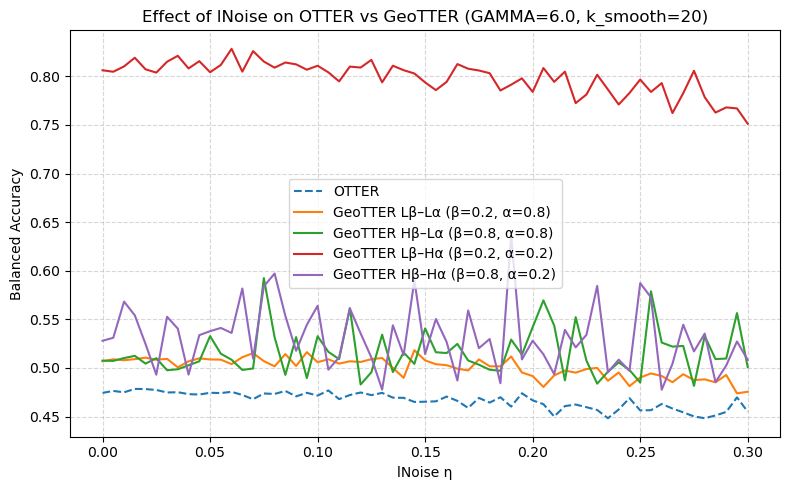}
    \caption{Synthetic data on lNoise}
  \end{subfigure}
  \hfill
  \caption{Noise on different ways of constructing gNoise and lNoise}
  \label{fig:triple_diagram}
\end{figure}

\textbf{Long-tail synthesis example setup.} We built a long-tailed, fully synthetic dataset to assess how OTTER and GeoOTTER behave under controlled noise. We first sampled $m=20$ prototype directions $p_c\sim\mathrm{Unif}(\mathbb S^{511})$. For the head classes ($c<10$) we split each prototype into three sub-clusters, whereas the tail classes ($c\ge10$) kept a single center. Every sub-cluster center was produced by adding a Gaussian offset $\delta_{c,s}\sim\mathcal N(0,\sigma_{\text{sub}}^{2}I)$ with $\sigma_{\text{sub}}=0.2$ and normalizing; we then generated $n_{c,s}$ samples around each center with an intra-cluster noise $\varepsilon_{c,s,i}\sim\mathcal N(0,\sigma_{\text{intra}}^{2}I)$ where $\sigma_{\text{intra}}=0.1$. The resulting angular variance with respect to $p_c$ was $\tau_\ell^{2}\approx\sigma_{\text{sub}}^{2}+\sigma_{\text{intra}}^{2}\simeq0.05$, and the head-to-tail sample ratio was roughly $10{:}1$.

\textbf{Impact of isotropic and local perturbations.} To examine robustness we injected two types of perturbation.  For global Gaussian noise we perturbed every vector by an isotropic term of amplitude $\sigma$, i.e.\ $\tilde x^{(g)}=\tfrac{x+\sigma z}{\lVert x+\sigma z\rVert}$ with $z\sim\mathcal N(0,I)$, and scanned $\sigma$ from $0$ to $0.50$ in increments of $0.05$.  For local orthogonal noise we first built a 10-NN graph on $\tilde x^{(g)}$, projected a fresh Gaussian vector onto the orthogonal complement of the local mean, and added a displacement of size $\eta$; we varied $\eta$ from $0$ to $0.30$ with step $0.03$.  Figure \ref{fig:triple_diagram} subfigure (a) plots the accuracy against $\sigma$ and $\eta$, respectively, while Figure \ref{fig:triple_diagram} subfigure (b) superimposes both sweeps.

\paragraph{Take-away.} We observed that the balanced accuracy of the OTTER baseline decreased almost linearly with either noise parameter, confirming that the effective logit variance grows as $\sigma^{2}$ or $\eta^{2}$.  In contrast GeoTTER maintained high accuracy: cluster filtering reduced the gNoise variance to $(\tfrac{\tau_\ell^{2}}{\sigma^{2}+\tau_\ell^{2}})\,\tfrac{\sigma^{2}}{n_\ell}$ and graph smoothing shrank the lNoise term by $\alpha_s^{2}+\tfrac{(1-\alpha_s)^{2}}{k}$, so the residual variance stayed an order of magnitude lower than that of OT.  Overall, the empirical trends in all three figures match the theoretical attenuation and illustrate the clear robustness margin enjoyed by GeoTTER when the noise level approaches the inner-product scale of the clean data.

\section{Implementation Details for Toy Example Figures}\label{app:Toy}
\textbf{Toy Clustering Fusion Setup} 

We provide a example of Clustering Fusion in Figure \ref{fig:overall}, which is obtained in five steps. (1) \textbf{Compute embeddings.} Let \(X \in \mathbb{R}^{N \times D}\) be the CLIP ViT-B/16 embeddings of the CIFAR-10 test set. 
(2) \textbf{Unsupervised clustering.} Apply K-means with \(k = 10\) to \(X\), yielding centroids \(C = \{c_1, \ldots, c_{10}\}\) and \emph{cluster assignments} \(\kappa(i) \in \{1, \ldots, 10\}\); these indices arise purely from the unsupervised clustering and are not the ground-truth class labels. 
(3) \textbf{Two-dimensional projection.} Fit one PCA map \(W \in \mathbb{R}^{D \times 2}\) to \(X\) and project both samples and centroids: \(Z = XW\) and \(CW\). 
(4) \textbf{Sub-sampling for legibility.} Retain at most \(M = 100\) projected samples per initial cluster. 
(5) \textbf{Compose the panel.} The left sub-plot scatters the subsampled \(Z_i\) as uniform light-grey squares and overlays the \emph{initial} K-means seeds \(C^{(0)}W\) as hollow black circles, visualising the geometry that underlies the original sample-wise cost matrix \(C_{ij} = 1 - \cos(x_i, y_j)\); the right sub-plot recolours the same points by \(\kappa(i)\), adds faint spokes from each \(Z_i\) to its \emph{final} centroid \(c_{\kappa(i)}W\), and marks those centroids as filled black-rimmed circles, thereby illustrating the cluster-level cost \(C^{\text{cluster}}_{ij} = 1 - \cos(c_{\kappa(i)}, y_j)\), under which every member of a cluster experiences the same directional influence in the fused cost matrix.

\begin{figure}[h]
  \centering
    \includegraphics[width=0.92\linewidth]{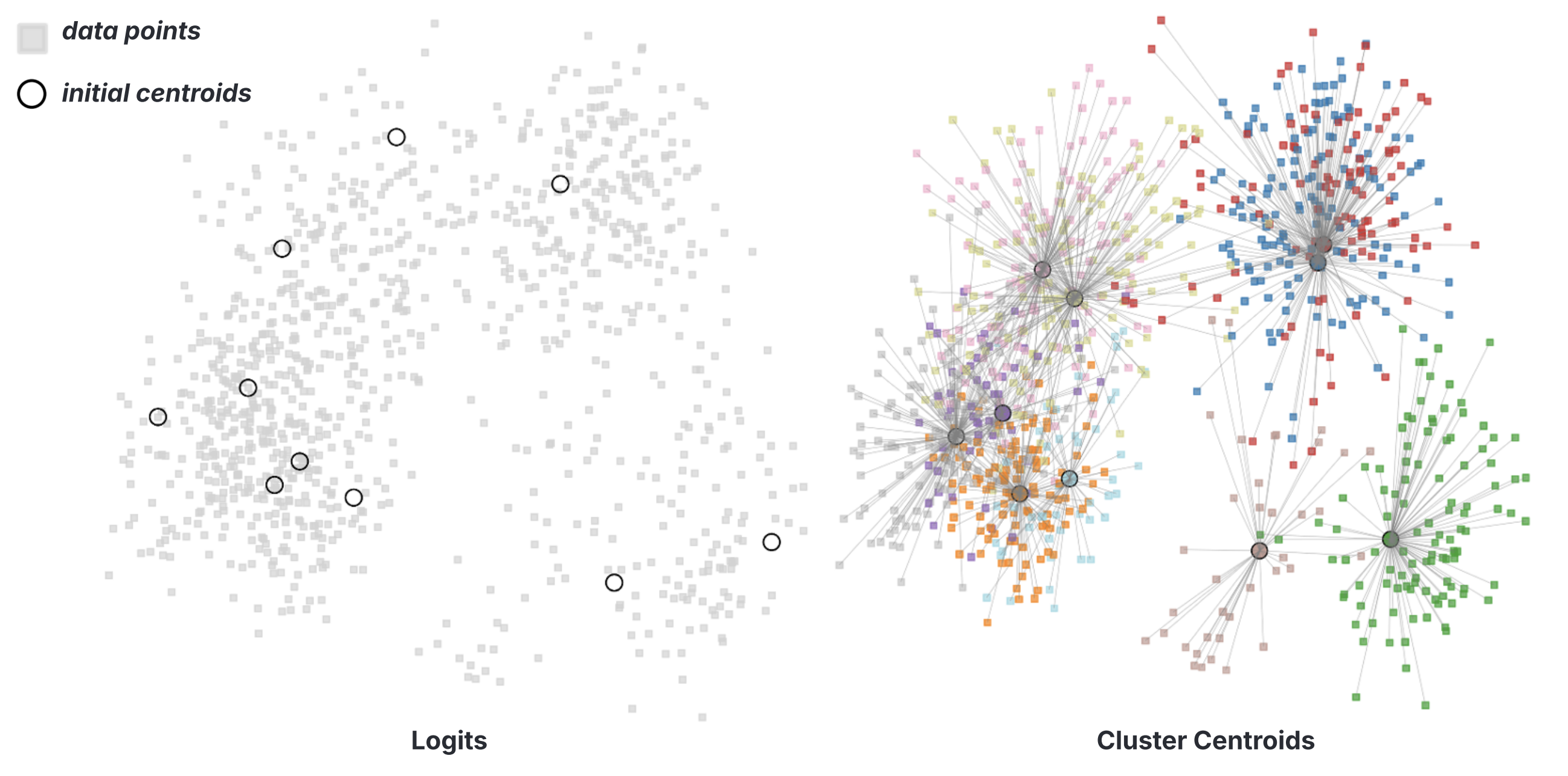}
      \captionsetup{width=0.92\columnwidth} %
  \caption{\textbf{Clustering-guided cost fusion on CIFAR-10.} 2-D projection of CIFAR-10 embeddings: grey points are individual samples, and the black-rimmed circles mark the $K=10$ centroids obtained with K-Means that are later used to build the cluster-level cost term.}
  \label{fig:overall}
\end{figure}

\textbf{Toy Swiss-Roll K-NN Setup} 

We also provide a example of Smoothing in Figure \ref{fig:smoothing}, which is generated in four direct steps. (1) \textbf{Swiss-roll point cloud.}  We draw $N=500$ samples $(u_i,v_i)$ with $u_i\!\sim\!\mathrm{Unif}[0,4\pi]$ and $v_i\!\sim\!\mathrm{Unif}[-15,15]$ and embed them as $x_i=(u_i\cos u_i,\,v_i,\,u_i\sin u_i)\in\mathbb R^{3}$; the roll angle $u_i$ is reused as the colour map. (2) \textbf{K-NN graph.}  For $k=10$ we compute Euclidean neighbours, build an undirected graph $G$ and plot every edge $(i,j)\in E$ as a faint grey segment. (3) \textbf{Left panel: “Constructing the K-NN Graph”.}  We scatter the coloured nodes together with the edge set from step (2), faithfully exposing the raw, unsmoothed neighbourhood structure. (4) \textbf{Right panel: “After Laplacian Graph Smoothing”.}  Without altering the node positions or recomputing labels we simply overlay the analytic Swiss-roll surface $S(u,v)$ as a translucent pastel mesh $(\alpha=0.3)$ on top of the same graph from step (2); this visual cue conveys the intended “smoothed” manifold while keeping the code minimal, the figure is illustrative only and does \emph{not} apply an explicit Laplacian operator.

\begin{figure}[H]
  \centering
    \includegraphics[width=0.92\linewidth]{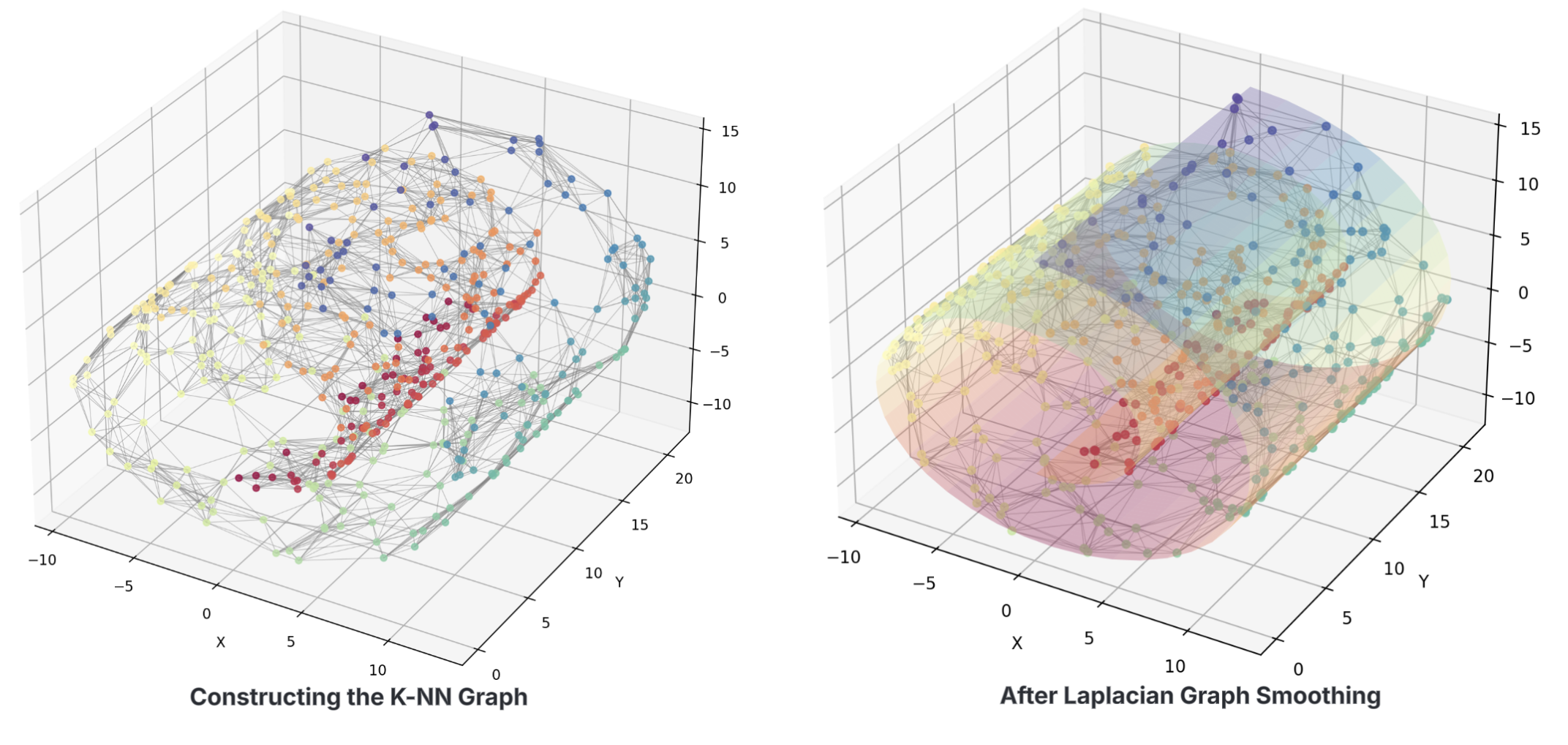}
  \captionsetup{width=0.92\columnwidth} %
  \caption{\textbf{Toy illustration of Laplacian graph smoothing.} Beginning with the $k$-nearest-neighbor graph constructed on the same embeddings (left), a few iterations of the Laplacian update in Eq.~(2) average each node over its neighborhood, eliminating high-frequency noise and producing a smoothly varying surface (right).}
  \label{fig:smoothing}
\end{figure}

\section{Statistics of datasets}\label{app:Statistics}

For each dataset listed, \emph{Min Shot} denotes the smallest number of labeled training examples provided per class under our few-shot protocol, while \emph{Max Shot} is the upper cap imposed on the number of examples per class so that classes with abundant data are subsampled to this limit for consistency. \emph{Total Samples} is the aggregate number of labeled training instances actually used in our experiments after the Min/Max constraints are applied across all classes. \emph{Num Classes} gives the count of mutually exclusive semantic categories present in the dataset. The first column simply reports the dataset name.

\begin{table}
\centering
\caption{Dataset statistics including min shots, max shots, sample size, and class count.}
\begin{tabular}{lccccc}
\hline
\textbf{\shortstack{Dataset}} & 
\textbf{\shortstack{Min\\Shot}} & 
\textbf{\shortstack{Max\\Shot}} & 
\textbf{\shortstack{Total\\Samples}} & 
\textbf{\shortstack{Num\\Classes}} \\
\hline
CIFAR10 & 1000 & 1000 & 10000 & 10 \\
CIFAR100 & 100 & 100 & 10000 & 100 \\
Caltech101 & 16 & 785 & 7162 & 101 \\
Caltech256 & 50 & 797 & 22897 & 257 \\
Country211 & 100 & 100 & 21100 & 211 \\
DTD & 40 & 40 & 1880 & 47 \\
EUROSAT & 1500 & 2500 & 22000 & 10 \\
Flowers102 & 20 & 238 & 6149 & 102 \\
Food101 & 250 & 250 & 25250 & 101 \\
OXFORDIIITPET & 88 & 100 & 3669 & 37 \\
STANFORDCARS & 24 & 68 & 8041 & 196 \\
STL10 & 800 & 800 & 8000 & 10 \\
SUN397 & 75 & 1907 & 87004 & 397 \\
CUB & 11 & 30 & 5794 & 200 \\
ImageNet & 40 & 40 & 40000 & 1000 \\
ImageNet-R & 31 & 410 & 26000 & 200 \\
ImageNet-Sketch & 40 & 41 & 40889 & 1000 \\
\hline
\end{tabular}
\label{tab:dataset_stats}
\end{table}

\section{Additional related work}\label{app:Related}

\subsection{Fenchel PM (FPM)}\label{app:Fenchel}

\paragraph{Overview of Proposition 4.1.}

The proposition 4.1 rewrites our regularized negative-entropy objective as a Fenchel-Prior Matching (FPM) saddle-point problem.By performing a three-step Fenchel biconjugation, we introduce an auxiliary probability vector p that absorbs the soft-max non-linearity.This reformulation yields a numerically stable objective whose inner maximizer admits a closed-form fixed-point equation and allows the outer weights w to be optimized with standard convex solvers.

\noindent\textbf{Proposition 4.1 (Three-step transformation to Fenchel-Prior Matching)}

\textit{Let}
\[
\Omega(p)=\sum_{j=1}^{m}p_j\log p_j ,\qquad p\in\Delta_m:=\bigl\{p\in\mathbb{R}^m\mid p_j\ge 0,\;\sum_{j}p_j=1\bigr\}.
\]
\textit{Fix logits} $\theta\in\mathbb{R}^{m}$, \textit{temperature} $T>0$, \textit{regularization weight} $\lambda>0$ \textit{and prior} $\nu\in\Delta_m$.
Define the quadratic attraction  
\[
R(p):=\frac{\lambda}{2}\,\|p-\nu\|_2^{2}.
\]
Then the optimization problem  

\[
\boxed{\;\displaystyle
\min_{w>0}\;\frac{\lambda}{2}\Bigl\|
      \varphi_{\Omega}\!\Bigl(\frac{w\odot\theta}{T}\Bigr)-\nu
      \Bigr\|^{2}\;}
\tag{P}
\]

admits the equivalent saddle-point form  

\[
\boxed{\;\displaystyle
\min_{w>0}\;\max_{p\in\Delta_m}
      \Bigl\langle\tfrac{w\odot\theta}{T},\,p\Bigr\rangle
      -\Omega(p)
      +\frac{\lambda}{2}\|p-\nu\|_2^{2}\;}
\tag{FPM}
\]

whose inner maximizer $p^\star(w)$ satisfies the fixed-point equation  

\[
p^\star
=\operatorname{softmax}\!\Bigl(\tfrac{w\odot\theta}{T}-\lambda\bigl(p^\star-\nu\bigr)\Bigr).
\]

\noindent\textbf{Proof of Proposition 4.1.}  
Start from the negative-entropy regularizer and introduce Lagrange multipliers $\lambda\in\mathbb{R}$ and $\mu\in\mathbb{R}^{m}_{\ge 0}$, maximizing
\[
\mathcal{L}(p,\lambda,\mu)
  =\sum_{j}\theta_j p_j-\sum_{j}p_j\log p_j
   \;-\;\lambda\Bigl(\sum_{j}p_j-1\Bigr)\;-\;\sum_{j}\mu_j p_j .
\]
First-order optimality with complementary slackness gives
\[
p_j^\star=\exp\bigl(\theta_j-\lambda-1-\mu_j\bigr),
\qquad
\mu_j p_j^\star =0
\;\;\Longrightarrow\;\;
\mu_j=0,
\]
so the optimum lies in the simplex interior and
\[
p_j^\star=\frac{e^{\theta_j}}{\sum_k e^{\theta_k}}
           =\operatorname{softmax}_j(\theta).
\]
Substituting $p^\star$ reveals the convex conjugate
\[
\boxed{\;\Omega^{*}(\theta)=\log\!\Bigl(\sum_{j}e^{\theta_j}\Bigr)\;}
,\qquad
\varphi_{\Omega}(\theta):=\nabla\Omega^{*}(\theta)=\operatorname{softmax}(\theta).
\]

For a target $y\in\Delta_m$, the corresponding Fenchel–Young loss is
\[
L_{\Omega}(\theta;y)
  =\Omega^{*}(\theta)+\Omega(y)-\langle\theta,y\rangle
  =\log\!\Bigl(\sum_{j}e^{\theta_j}\Bigr)
   -\sum_{j}y_j\log y_j
   -\sum_{j}\theta_j y_j .
\tag{FY}
\]
When $y$ is one-hot, $\Omega(y)=0$ and (FY) reduces to the familiar cross-entropy $-\log\bigl(\operatorname{softmax}_y(\theta)\bigr)$.

Next, augment the objective by a quadratic prior term: injecting $R$ and treating the positive weights $w$ as optimization variables yields the primal problem (P) stated earlier. Direct gradient descent on (P) is numerically fragile because the softmax map is sharply curved along directions orthogonal to the simplex.

To obtain a stable formulation, perform Fenchel biconjugation. Introduce an auxiliary $p\in\Delta_m$ and define
\[
f(w)=
\min_{p\in\Delta_m}\Bigl[
      D_{\Omega}\!\Bigl(p\;\Big\|\,
        \varphi_{\Omega}\!\bigl(\tfrac{w\odot\theta}{T}\bigr)\Bigr)
      +\frac{\lambda}{2}\,\|p-\nu\|_2^{2}\Bigr],
\]
where
$
D_{\Omega}(p\|q)
   =\Omega(p)+\Omega^{*}(\tfrac{w\odot\theta}{T})
    -\langle\tfrac{w\odot\theta}{T},p\rangle
$.
Since $\Omega^{*}(\tfrac{w\odot\theta}{T})$ does not depend on $p$, we can drop it and swap $\min$–$\max$ order, arriving at the dual game (FPM). Its maximiser $p^\star(w)$ satisfies
\[
\frac{w\odot\theta}{T}-\nabla\Omega(p^\star)
      +\lambda\bigl(p^\star-\nu\bigr)=0,
\]
and, using $\nabla\Omega(p)=-1-\log p$, we recover the fixed-point equation stated in the proposition.

Because the inner optimization over $p$ is \emph{strongly concave} (strict convexity of $\Omega$ plus the quadratic term), $p^\star(w)$ is unique and differentiable. Eliminating $p$ therefore leaves a smooth, strongly convex outer landscape in $w$, so standard solvers such as Adam or L-BFGS converge reliably, avoiding the numerical brittleness observed in the primal formulation.

\paragraph{Overview of Proposition 4.2.}

Having reduced learning to the outer problem 
\[
g(w)=\tfrac{\lambda}{2}\lVert p^{\star}(w) - \nu \rVert_2^2,
\]
we now ask: \textit{Does $g$ admit strong convexity and Lipschitz-gradient constants sufficient for linear convergence of gradient descent.} The next proposition establishes strong convexity and smoothness constants that guarantee:  
(i) uniqueness of the global minimizer,  
(ii) bounded gradients, and  
(iii) linear convergence of gradient descent with an explicit rate depending solely on the condition number $\kappa$.

\noindent\textbf{Proposition 5.1 \;(First-order geometry of the outer FPM objective).}
Let
\[
g(w)=\frac{\lambda}{2}\,\bigl\|p^{\star}(w)-\nu\bigr\|_2^{2},\qquad
p^{\star}(w)=\argmax_{p\in\Delta_m}
        \Bigl\langle\tfrac{w\odot\theta}{T},p\Bigr\rangle-\Omega(p),
\]
with temperature $T>0$, regularization weight $\lambda>0$, and prior $\nu\in\Delta_m$.
Set $\underline{\theta}:=\min_j|\theta_j|>0$ and $\overline{\theta}:=\max_j|\theta_j|$.
After eliminating the inner maximization, we obtain the following facts.

\medskip
\noindent\textbf{Lemma 1 (Strong convexity).}
For all $x,y$,
\[
g(y)\;\ge\;g(x)+\langle\nabla g(x),y-x\rangle
          +\frac{\mu}{2}\,\|y-x\|_2^{2},
\qquad
\mu=\dfrac{\lambda\,\underline{\theta}^{2}}{T^{2}}.
\]
Thus $g$ has a \emph{unique} global minimizer; “flat-corner’’ stationary points cannot occur \citep{nesterov2004}.

\emph{Proof.}
The inner problem is $1$-strongly concave in $p$, so $p^{\star}(w)$ depends affinely on $\tfrac{w\odot\theta}{T}$.
By the envelope theorem this transfers to strong convexity of $g$ with modulus $\mu=\lambda\underline{\theta}^{2}/T^{2}$.
\hfill$\square$

\medskip
\noindent\textbf{Lemma 2 (Smoothness).}
For all $x,y$,
\[
g(y)\;\le\;g(x)+\langle\nabla g(x),y-x\rangle
          +\frac{L}{2}\,\|y-x\|_2^{2},
\qquad
L=\dfrac{\lambda\,\overline{\theta}^{2}}{T^{2}}.
\]
Consequently, any fixed stepsize $\eta\in(0,2/L)$ yields monotone descent \citep{boyd2004}:
\[
g(w_{t+1})\;\le\;g(w_t)
      -\frac12\Bigl(\tfrac1\eta-\tfrac{L}{2}\Bigr)\|w_{t+1}-w_t\|_2^{2},
\]
preventing overshoot and sign-flip when $\eta<2/L$.

\emph{Proof.}
Because $p^{\star}(w)$ is affine in $\tfrac{w\odot\theta}{T}$, the Hessian of $g$ is bounded above by $L=\lambda\overline{\theta}^{2}/T^{2}$.
\hfill$\square$

\medskip
\noindent\textbf{Lemma 3 (Polyak–Łojasiewicz inequality).}
Combining Lemmas 1 and 2 gives
\[
\frac12\bigl\|\nabla g(w)\bigr\|_2^{2}
   \;\ge\;\mu\bigl(g(w)-g^{\star}\bigr),
\]
so a small gradient norm arises only when $g(w)$ is already near its minimum \citep{karimi2016}.

\emph{Proof.}
Strong convexity implies
$\langle\nabla g(w),w-w^{\star}\rangle\ge\mu\|w-w^{\star}\|_2^{2}$,
while smoothness gives
$\|\nabla g(w)\|_2\le L\|w-w^{\star}\|_2$.
Multiplying and rearranging yields the inequality.
\hfill$\square$

\medskip
\noindent\textbf{Definition (Condition number).}
\[
\kappa:=\frac{L}{\mu}
        =\Bigl(\frac{\overline{\theta}}{\underline{\theta}}\Bigr)^{2}.
\]

\medskip
\noindent\textbf{Take-away.}
\begin{enumerate}
  \item \textbf{No spurious stationary points}: if $\nabla g(w)=0$, then $w=w^{\star}$.
  \item \textbf{Co-coercive gradient}: step norms remain bounded.
  \item \textbf{Linear convergence}: gradient descent with $\eta\le1/L$ contracts at rate $1-1/\kappa$ \citep{bubeck2015}.
\end{enumerate}

\end{document}